\documentclass[11pt]{article}
\usepackage{amsfonts}       
\usepackage{amsmath}
\usepackage{amsthm}
\usepackage{amsfonts}
\usepackage{amssymb}
\usepackage{array}
\usepackage{booktabs}       
\usepackage{caption}
\usepackage{comment}        
\usepackage{enumitem}
\usepackage[utf8]{inputenc} 
\usepackage{float}
\usepackage[T1]{fontenc}    
\usepackage[a4paper,
            bindingoffset=0.0in,
            left=1in,
            right=1in,
            top=1in,
            bottom=1.25in,
            footskip=.5in]{geometry}
\usepackage{hyperref}       
\usepackage{multirow}
\usepackage{natbib}
\usepackage{nicefrac}       
\usepackage{microtype}      
\usepackage{subcaption}
\usepackage{xcolor}         
\usepackage{ulem}
\usepackage{url}            
\usepackage{tcolorbox}
\usepackage{xfrac}

\newcommand{\bI}{\mathbf{I}}
\newcommand{\bM}{\mathbf{M}}
\newcommand{\bX}{\mathbf{X}}
\newcommand{\bY}{\mathbf{Y}}
\newcommand{\bZ}{\mathbf{Z}}
\newcommand{\bA}{\mathbf{A}}
\newcommand{\bB}{\mathbf{B}}
\newcommand{\bC}{\mathbf{C}}
\newcommand{\bD}{\mathbf{D}}
\newcommand{\bF}{\mathbf{F}}

\newcommand{\bS}{\mathbf{S}}

\newcommand{\bx}{\mathbf{x}}
\newcommand{\bz}{\mathbf{z}}

\newcommand{\bu}{\mathbf{u}}

\newcommand{\bv}{\mathbf{v}}

\newcommand{\bTheta}{\boldsymbol{\Theta}}
\newcommand{\0}{\mathbf{0}}
\newcommand{\1}{\mathbf{1}}
\newcommand*{\QEDB}{\hfill\ensuremath{\square}}
\DeclareRobustCommand{\rchi}{{\mathpalette\irchi\relax}}
\newcommand{\irchi}[2]{\raisebox{\depth}{$#1\chi$}} 
\newcommand{\argmin}{\operatornamewithlimits{arg\,min}}
\newcommand{\argmax}{\operatornamewithlimits{arg\,max}}

\newtheorem{theorem}{Theorem}
\newtheorem{lemma}{Lemma}
\newtheorem{corollary}{Corollary}
\newtheorem{remark}{Remark}%

\title{\textbf{Classification of High-dimensional Time Series in Spectral Domain using Explainable Features}}

\author{%
  Sarbojit Roy
  \\
  \texttt{\small sarbojit.roy@kaust.edu.sa} \\
   and \\
   Malik Shahid Sultan \\
  \texttt{\small malikshahid.sultan@kaust.edu.sa} \\
   and \\
   Hernando Ombao \\
  \texttt{\small hernando.ombao@kaust.edu.sa}\\ \\
  \small Computer, Electrical, and Mathematical Sciences and Engineering\\
  \small King Abdullah University of Science and Technology\\
  \small Saudi Arabia -- 23955
}

\begin{document}

\maketitle

\begin{abstract}
Interpretable classification of time series presents significant challenges in high dimensions. Traditional feature selection methods in the frequency domain often assume sparsity in spectral density matrices (SDMs) or their inverses, which can be restrictive for real-world applications. In this article, we propose a model-based approach for classifying high-dimensional stationary time series by assuming sparsity in the difference between inverse SDMs. Our approach emphasizes the interpretability of model parameters, making it especially suitable for fields like neuroscience, where understanding differences in brain network connectivity across various states is crucial. The estimators for model parameters demonstrate consistency under appropriate conditions. We further propose using standard deep learning optimizers for parameter estimation, employing techniques such as mini-batching and learning rate scheduling. Additionally, we introduce a method to screen the most discriminatory frequencies for classification, which exhibits the sure screening property under general conditions. The flexibility of the proposed model allows the significance of covariates to vary across frequencies, enabling nuanced inferences and deeper insights into the underlying problem. The novelty of our method lies in the interpretability of the model parameters, addressing critical needs in neuroscience. The proposed approaches have been evaluated on simulated examples and the `Alert-vs-Drowsy' EEG dataset.
\end{abstract}

\section{Introduction}\label{intro}

Let $\{X(t)=\left (\mathrm{X}_{1}(t), \ldots, \mathrm{X}_{p}(t)\right )^\top,1\leq t\leq T\}$ be a $p$-dimensional Gaussian stationary time series of length $T$ and $Y\in \{1,2\}$ be the class label of associated with $\bX=[X(1), \ldots, X(T)]\in \mathbb{R}^{p\times T}$. We assume that the series is centered, i.e., $\mathrm{E}[X(t)]=0_p$ for all $t$. The class densities are denoted by $f_1$ and $f_2$, i.e.,  $\bX|Y=1\sim f_1$ and $\bX|Y=2\sim f_2$. Also, let $(\bX_l,Y_l)$ be the independent and identically distributed (iid) copies of $(\bX,Y)$, for $1\leq l\leq n$. The class priors are given by $0<\mathrm{P}[Y=j]=\pi_j<1$ for $j=1,2$ with $\pi_1+\pi_2=1$. Our aim is to predict the unknown class label of a time series $\bz = [z(1), \ldots, z(T)]$ with $z(t)= \left (\mathrm{z}_1(t), \ldots, \mathrm{z}_p(t)\right )^\top$,  $1\leq t\leq T$. If $f_1, f_2, \pi_1$ and $\pi_2$ are known, then the optimal classifier, i.e., the Bayes classifier assigns $\bz$ to class 1 if $\ln{\pi_1}+\ln{f_1(\mathbf{z})}>\ln{\pi_2}+\ln{f_2(\mathbf{z})}$, and to class 2, otherwise.
In this article, we focus on classifying time series in the spectral domain. In brain network analysis, it is essential to understand how connectivity patterns between brain regions differ across states, such as drowsy and alert. These patterns can vary by frequency; for example, high-frequency EEG components might be similar, while low-frequency components differ significantly. Thus, classifying in the spectral domain offers three main advantages: (a) identifying relevant frequencies enhances understanding of brain connectivity, (b) removing irrelevant frequencies improves decision rule performance, and (c) frequency-specific connectivity information provides deeper insights into brain region associations. This approach to classification using spectral information is, of course, not limited to neuroscience, and can be applied to various other fields, e.g., economic panel data, and systems biology.

In the spectral domain, the log densities can be approximated by the Whittle log-likelihood:
\begin{gather}\label{whittleapprox}
   \ln{f_l(\mathbf{z})}\approx W_l(\bz) = \sum\limits_{\omega_k\in \Omega_T}\left [\ln{|{\bTheta_l}(\omega_k)|} -z^*(\omega_k){\bTheta_l}(\omega_k)z(\omega_k)\right ], 
\end{gather}
where $ \Omega_T=\{\omega_k=k/T,\ k=1,\ldots, T^\prime\}$ with $T^\prime = [T/2]-1$ is the set of Fourier frequencies, $z(\omega_k)\in \mathbb{C}^p$ is Discrete Fourier Transform (DFT) of a series $z(t)$ at frequency $\omega_k$ and $\bTheta_l(\omega)\in \mathbb{C}^{p\times p}$ is the inverse of spectral density matrix (SDM) of class $l$ for $l=1,2$ \citep[see, e.g., ][]{shumway2000time}. We denote the SDMs of class $l$ by $\bS_l(\omega_k)$, for $l=1,2$. Here $z^*$ denotes the conjugate transpose of $z\in \mathbb{C}^p$. A classifier based on the Whittle approximation is defined as
\begin{gather}\label{class_BW}
    \delta_\text{W}(\mathbf{z}) = \begin{cases}
    1,&\mathcal{L}(\bz)=\ln {\left (\pi_1/\pi_2\right )} + W_1(\mathbf{z})-W_2(\mathbf{z})>0,\\
    2,&\text{ otherwise}.
    \end{cases}
\end{gather}
It is clear from equations \eqref{whittleapprox} and \eqref{class_BW} that if $\bTheta_1(\omega_k)=\bTheta_2(\omega_k)$ for some $\omega_k\in\Omega_T$, then the frequency $\omega_k$ plays no role in classification. Such frequencies, collectively denoted by the set
\begin{align}\label{nset}
    \Omega^0_T=\{\omega_k: \bTheta_1(\omega_k)=\bTheta_2(\omega_k), \omega_k\in\Omega_T\},
\end{align}
are referred to as `noise'. On the other hand, the set $\Omega^\bD_T = \Omega_T\setminus \Omega^0_T$ contains all necessary information relevant for the classification, and is referred to as `signal'. Estimating $\Omega^0_T$ and $\Omega^\bD_T$ helps us gain valuable insight into how connectivity pattern varies across different states of the brain. Subsequently, removing the irrelevant frequencies will further improve the accuracy of $\delta_W$.

A plugin-based estimate of the discriminant $\mathcal{L}(\bz)$ is given by $\hat{\mathcal{L}}(\bz) = \ln {\left (n_1/n_2\right )} + \hat{W}_1(\mathbf{z})-\hat{W}_2(\mathbf{z})$ where $\hat{W}_l(\mathbf{z})$ is obtained by simply plugging the estimates  $\hat{\bTheta}_l(\omega_k)$ in \eqref{whittleapprox} for $\omega_k\in \Omega_T$. However, in high dimensions, estimation of the inverse SDMs is challenging, since the sample SDMs become ill-conditioned when $p$ is large. One needs additional assumptions on the structure of $\bTheta_l(\omega_k)$ for its consistent estimation. \citep[][]{ledoit2020analytical, sun2018large,
fiecas2017shrinkage,14-EJS977,Pourahmadi}{}{}. \cite{14-EJS977} proposed a data-driven shrinkage method to obtain consistent estimators of SDMs in high dimensions. Another common approach is to assume sparsity on the matrix $\bTheta_l(\omega_k)$ for $l=1,2$ \citep[see, e.g., ][]{l1min_TcaiJASA2011, MR4017528}.  
Recall that the $j_1,j_2$-th element of $\bTheta_l(\omega_k)$ captures the conditional dependence between the $j_1$-th and $j_2$-th covariate given the rest, at frequency $\omega_k$. The sparsity of $\bTheta_l(\omega_k)$ implies that only a few among the $p$ covariates are conditionally dependent on each other at frequency $\omega_k$. Such an assumption can be restrictive and inappropriate in many real-world situations, such as neurological data where different regions in the brain are known to have a dense network. Moreover, the existing methods are primarily interested in obtaining a consistent estimator of the inverse SDM in high dimensions. On the other hand, we aim to obtain estimates that minimize the classification error. It is well known that consistent estimation does not always lead to improved classification accuracy \citep[see, e.g., ][]{cai2011direct, mai2012direct}. This motivates us to develop a method of estimating the model parameters from a different point of view. Instead of sparsity in $\bTheta_1(\omega_k)$ and $\bTheta_2(\omega_k)$ for each $\omega_k\in\Omega_T$, we propose a model that only requires the difference between inverse SDMs, i.e., $\{\bTheta_2(\omega_k)-\bTheta_1(\omega_k)\}$ to be sparse. 
Observe that the set $\Omega^\bD_T$ can also be expressed as $\Omega^\bD_T=\{\omega_k:\bTheta_1(\omega_k)\neq \bTheta_2(\omega_k),\omega_k\in\Omega_T \}$, and the discriminant $\mathcal{L}(\bz)$ in \eqref{class_BW} can be written as 
\begin{align}\label{disc0}
    \mathcal{L}(\bz)=\ln {\frac{{\pi}_1}{{\pi}_2}}+ \sum\limits_{\omega_k\in {\Omega}^{\bD}_T}\left [{z}(\omega_k)^*{{\bD}}(\omega_k) {z}(\omega_k)+\ln{|{{\bD}}(\omega_k){{\bS}}_{1}(\omega_k) + \bI|} \right ],
\end{align}
where $\bD(\omega_k) = \bTheta_2(\omega_k)-\bTheta_1(\omega_k)$ at frequency $\omega_k\in \Omega^\bD_T$ and $\bI$ is the $p\times p$ identity matrix. It is clear from equation \eqref{disc0} that the $\mathcal{L}(\bz)$ depends on the unknown parameters only through the matrices $\bD_k$ and $\bS_{1k}$ for $\omega_k\in\Omega^\bD_T$.  
Expressing the discriminant \(\mathcal{L}(\bz)\) this way has two advantages: (a) \textit{Interpretability}: In neurological data, brain states like 'alert' and 'drowsy' differ in specific frequency bands. Our framework identifies these frequencies in \(\Omega^\bD_T\). The parameter \(\bD(\omega_k), \omega_k \in \Omega_T\) indicates the difference in dependence between covariates, where the \((j,l)\)-th element shows the difference between the \(j\)-th and \(l\)-th covariates. (b) \textit{Efficiency}: We avoid inverting high-dimensional sample SDMs by directly estimating \(\bD_k\) matrices. Furthermore, this approach allows the sparsity pattern to vary across frequencies, highlighting which covariates contribute to differences at each frequency \(\omega_k\).

Among feature-based classification methods, Shapelet-based methods are popular for their interpretability in feature-based classification but scale poorly in multivariate time series classification (MTSC). \citet{karlsson2016generalized} introduced a scalable MTSC method using randomly selected shapelets in decision tree forests. Recent shapelet-based methods like the Shapelet Transform Classifier (STC) and Random Interval Spectral Ensemble focus on univariate series and can be combined into ensembles such as HIVE-COTE, but these aren't truly multivariate as components operate independently on each dimension. Canonical interval forest (CIF) uses time-series trees but lacks interpretability. An interpretable method by \citet{le2019interpretable} uses symbolic representation but is limited to linear classifiers. MiniROCKET and ROCKET achieve high accuracy but sacrifice interpretability, as do deep learning methods like InceptionTime and TapNet. \citet{ruiz2021great} analyzed several algorithms using the UEA multivariate time series classification archive 2018 and found HIVE-COTE, CIF, ROCKET, and STC significantly more accurate than the dynamic time warping algorithm, but except for ROCKET, the other methods have long training times and as mentioned above, all four lack interpretability. Hence, there is a need for interpretable MTSC methods, especially for high-dimensional series.

\vspace{0.5cm}
\noindent \textbf{Our contribution}: 
We propose a method for classifying high-dimensional stationary Gaussian time series in the frequency domain and develop a procedure for identifying the most relevant frequencies for classification. The main advantages of our method are:

 \begin{itemize}
    \item \textit{Interpretable model parameters}: We believe, our method is the first method that leads to interpretable results in the regime of high-dimensional time series classification.
    \item \textit{Consistent estimation}: We develop theoretical results that hold in the ultrahigh-dimensional settings when the dimension may increase non-polynomially with the sample size. 
     \item \textit{Sparse difference assumption}: Unlike traditional methods, we assume sparsity in the difference between inverse spectral density matrices (SDMs) rather than in the SDMs themselves, achieving theoretical consistency under less restrictive conditions.
     \item \textit{Effective frequency screening}: Our frequency screening method enjoys the {\it sure screening property} \citep{fan2008sure}, and ranks frequencies by their importance in classification. 
    \item \textit{Flexible framework}: The method allows the importance of covariates to vary across frequencies, providing deeper insights into real-world problems such as brain network connectivity.
\end{itemize}
In Section \ref{method}, we describe the estimation procedure. The theoretical properties of the proposed estimators are presented in Section \ref{thprop}. We demonstrate the performance of the classifier on simulated and real data sets in Sections \ref{simstud} and \ref{real}, respectively. We end this article with a discussion on the limitations of the proposed method and potential directions for future research in Section \ref{limit}. All proofs are detailed in the Appendix \ref{appen}.  
\section{Methodology}\label{method}

{\it Notations and definitions:} The symbol $\mathrm{i}$ denotes the square root of $(-1)$. For an integer $p\geq 1$, $[p]$ denotes the set $\{1,\ldots, p\}$. Given a vector $\bv\in \mathbb{R}^p$, let $\|\bv\|_1$, $\|\bv\|_2$ and $\|\bv\|_\infty$ denote its $\ell_1$, $\ell_2$ and $\ell_\infty$ norms, respectively. We use $supp(\bv)$ to denote the support of $\bv$. For a $p\times p$ matrix $\bM$, we write $\|\bM\|_\mathrm{F}=\sum_{i,j}^p (A^2_{i,j})^{1/2}$ for its Frobenius norm, $\|\bM\|_0 = \sum_{i,j}\mathbb{I}[\bM(i,j)\neq 0]$, $\|\bM\|_1 = \sum_{i,j}|\bM(i,j)|$ and $\|\bM\|_\infty = \max_{i,j}|\bM(i,j)|$. In addition, we define $\|\bM\|_{1,\infty} = \max_i\sum_{j}|\bM_{i,j}|$ to be the $\ell_{1,\infty}$ norm of the matrix $\bM$. Let $\lambda_j(\bM)$ denote the $j$-th eigen value of $\bM$ with $\lambda_1(\bM)\geq \cdots \geq \lambda_p(\bM)$. $\bM\succ 0$ implies that $\bM$ is a positive definite matrix. We further use $Vec(\bM)$ for the $p^2$-dimensional vector obtained by stacking the columns of $\bM$, and $\bM_1\bigotimes\bM_2$ for the Kronecker product between $\bM_1$ and $\bM_2$. We use the symbol $\0_p$ to denote both the $p\times 1$ zero vector and the $p\times p$ zero matrix, depending on the context. The $p\times p$ identity matrix is denoted by $\bI_p$. $\mathbb{I}[A]$ denotes the indicator of $A$. The subscript $p$ is sometimes dropped for brevity when the dimension is clear. For complex-valued vector $Z = Z^\mathcal{R} + \mathrm{i}Z^\mathcal{I}\in\mathbb{C}^p$, and matrix $\bM = \bM^\mathcal{R}+\mathrm{i}\bM^\mathcal{I}\in\mathbb{C}^{p\times p}$, we define 
\begin{gather*}
    \tilde{Z} = \begin{bmatrix}
        Z^\mathcal{R}\\
        Z^\mathcal{I}
    \end{bmatrix}\in \mathbb{R}^{2p},\text{ and }
    \widetilde{\bM}=\begin{bmatrix}
            \bM^\mathcal{R}&\bM^\mathcal{I}\\
            -\bM^\mathcal{I}&\bM^\mathcal{R}
        \end{bmatrix}\in \mathbb{R}^{2p\times 2p}.
\end{gather*}

We use $c,c_1,c_2,\ldots, C, C_1, C_2,\ldots,$ to denote the constants that do not depend on $n,p, T$, and their values may vary from place to place throughout this article.
To reduce the notational complexity, we write ${z}(\omega_k)$, ${\bS}_{l}(\omega_k)$, ${\bTheta}_{l}(\omega_k)$, and ${\bD}(\omega_k)$ as ${z}_k$, ${\bS}_{lk}$, ${\bTheta}_{lk}$, and ${\bD}_{k}$, respectively, from here onwards. With the notations introduced above, the discriminant $\mathcal{L}(\bz)$ in equation \eqref{disc0} can be written as
\begin{align}\label{disc_fin}
    \mathcal{L}(\bz)=\ln {\frac{{\pi}_1}{{\pi}_2}}+ \frac{1}{2}\sum\limits_{\omega_k\in {\Omega}^{\bD}_T}\left [\tilde{z}_k^\top{\widetilde{\bD}}_k \tilde{z}_k+\ln{|{\widetilde{\bD}}_k{\widetilde{\bS}}_{1k} + \bI|} \right ].
\end{align}
Now, we introduce our methods for estimating the model parameters, focusing on identifying the set of relevant frequencies \(\Omega^\bD_T\) and the difference matrices \(\widetilde{\bD}_k\) for all \(\omega_k \in \Omega^\bD_T\). We present two distinct approaches to estimate \(\widetilde{\bD}_k\) for \(\omega_k \in \hat{\Omega}^\bD_T\). In Section \ref{est_Domega}, we formulate the estimation as a problem of minimizing a convex loss function. Then, in Section \ref{estnn}, we modify this convex loss function to simultaneously minimize classification errors. The adjusted loss function includes a non-convex component and is solved using the Adam optimization algorithm \citep{kingma2014adam}. Lastly, in Section \ref{freqscreen}, we utilize the estimates of \(\widetilde{\bD}_{k}\) to develop a screening procedure and derive \(\hat{\Omega}^\bD_T\).
\subsection{Estimation of difference matrices $\widetilde{\bD}_k$}\label{est_Domega}
Fix $\omega_k\in \Omega_T$, and recall that $\widetilde{\bD}_k=\widetilde{\bTheta}_{2k}-\widetilde{\bTheta}_{1k}$. One way to estimate $\widetilde{\bD}_k$ is to estimate the inverse matrix $\widetilde{\bTheta}_{lk}$ for $l=1,2$, and consider $\hat{\widetilde{\bD}}_k$ to be the difference of the estimated matrices. 
\citet{cai2011constrained} developed a method to estimate high-dimensional precision matrices based on constrained $l_1$-minimization (CLIME). \citet{MR4017528} and \citet{krampe2022frequency} proposed CLIME-type estimators of a precision matrix in the spectral domain and studied its finite sample behavior. However, the consistency of these estimators is achieved under sparsity which may not be a realistic assumption in neurological data.  
Motivated by the work of \citet{yuan2017differential}, we take a direct approach to estimate the matrix $\widetilde\bD_k$. Instead of assuming sparsity on $\widetilde{\bTheta}_{lk}$ for $l=1,2$, we assume that their difference $\widetilde\bD_k = \widetilde{\bTheta}_{2k}-\widetilde{\bTheta}_{1k}$ is sparse. 
Observe that $\widetilde{\bS}_{1k}\widetilde{\bD}_k\widetilde{\bS}_{2k} = \widetilde{\bS}_{1k} - \widetilde{\bS}_{2k}\text{ and }\widetilde{\bS}_{2k}\widetilde{\bD}_k\widetilde{\bS}_{1k} = \widetilde{\bS}_{1k} - \widetilde{\bS}_{2k}$, i.e., 
$(\widetilde{\bS}_1\widetilde{\bD}_k\widetilde{\bS}_{2k}+\widetilde{\bS}_{2k}\widetilde{\bD}_k\widetilde{\bS}_{1k})/2 = \widetilde{\bS}_{1k} - \widetilde{\bS}_{2k}.$
We consider the following {\it D-trace loss} function:
\begin{align}\label{dtraceloss}
        & L(\widetilde{\bD}_k; \widetilde{\bS}_{1k}, \widetilde{\bS}_{2k}) = \frac{1}{4}\left [\langle\widetilde{\bS}_{1k}\widetilde{\bD}_k,\widetilde{\bD}\widetilde{\bS}_{2k}\rangle+\langle\widetilde{\bS}_{2k}\widetilde{\bD}_k,\widetilde{\bD}_k\widetilde{\bS}_{1k}\rangle\right ] - \langle\widetilde{\bD}_k,\widetilde{\bS}_{1k}-\widetilde{\bS}_{2k}\rangle,
\end{align}
where $\langle\bA,\bB\rangle=tr(\bA\bB^\top)$.
The above loss is convex with respect to $\widetilde{\bD}_k$ and has a unique minimum satisfying $\widetilde{\bD}_k = \widetilde{\bTheta}_{2k} - \widetilde{\bTheta}_{1k}$ for all $\omega_k\in\Omega^\bD_T$. We obtain sparse estimates of $\widetilde{\bD}_k$ by minimizing a lasso penalized $L(\widetilde{\bD}_k, \widetilde{\bS}_{1k}, \widetilde{\bS}_{2k})$. We first estimate the SDMs for the two classes. For a given training sample $\rchi_n=\{(\bX_1, Y_1),\ldots,(\bX_{n}, Y_n)\}$, let $\hat{\bS}^j_{k}$ denote the smoothed periodogram estimator of SDM based on $\bX_j$ for $j=1,\ldots, n$\ \citep[see, e.g., ][]{shumway2000time}. We define $\hat{\bS}_{lk} = \sum_j\mathbb{I}[Y_j=l]\hat{\bS}^j_{k}/n_1$ and $\hat{\widetilde{\bS}}_{lk} = \widetilde{\hat{\bS}}_{lk}$ for $l=1,2$ and $k=1,\ldots, T^\prime$. We estimate $\widetilde{\bD}_k$-s by minimizing the following:
\begin{align}\label{lassopen}
     \min\limits_{\widetilde{\bD}_1,\ldots, \widetilde{\bD}_{T^\prime}}\left \{\sum\limits_{\omega_k\in {\Omega}^\bD_T}L(\widetilde{\bD}_k; \hat{\widetilde{\bS}}_{1k}, \hat{\widetilde{\bS}}_{2k})+\lambda\sum\limits_{\omega_k\in {\Omega}^\bD_T}\|\widetilde{\bD}_k\|_1\right \},
\end{align} where $\lambda>0$ is a tuning parameter. This method is motivated by the sparse quadratic discriminant method by \citet{yuan2017differential} where the authors estimated the parameters in the context of differential graph estimation. The minimization problem can be solved using suitable off the shelf convex optimization solver. 
We used Adaptive Moment Estimation \cite{kingma2014adam} popularly known as ADAM as the optimizer to estimate the $\widetilde{\bD}_k$. We refer to Section \ref{estnn} for the details of implementation.
The tuning parameter $\lambda$ in \eqref{lassopen} is selected by minimizing the Generalized Information Criterion (GIC) proposed by \citet{gickim2012}. If $\Lambda_n$ is a sequence of $\lambda$ values, then we define.
\begin{align}\label{gic}
    \{\hat{\widetilde{\bD}}_{1},\ldots, \hat{\widetilde{\bD}}_{T^\prime}\}=\argmin\limits_{\lambda\in \Lambda_n}\frac{1}{n}\sum\limits_{\omega_k\in {\Omega}^\bD_T}\left \{L(\hat{\widetilde{\bD}}_k(\lambda), \hat{\widetilde{\bS}}_{1k},\hat{\widetilde{\bS}}_{2k}) + \log{(\log{n})}\log{p^2}\|\hat{\widetilde{\bD}}_k(\lambda)\|_0\right \},
\end{align}
where $\hat{\widetilde{\bD}}_k(\lambda)$ is the estimate of $\widetilde{\bD}_k$ for a given value of $\lambda$.

\subsection{Estimation of difference matrices using ADAM}\label{estnn}
Recall that our aim is to minimize the classification error of $\delta_W$ defined in \ref{class_BW}. Keeping that in mind, we further develop a method to simultaneously obtain estimates of $\widetilde{\bD}_k$. Let the posterior probability of class $l$ be denoted by $\pi_l(\bx) =\mathrm{P}[Y=l\mid \bx]$ for an observed time series $\bx\in \mathbb{R}^{p\times T}$. The probability $\pi_1(\bx)$ is defined as
\begin{align*}
    \pi_1(\bx)= \frac{\pi_1f_1(\bx)}{\pi_1f_1(\bx)+\pi_2f_2(\bx)}=\frac{\pi_1f_1(\bx)/\pi_2f_2(\bx)}{1+\{\pi_1f_1(\bx)/\pi_2f_2(\bx)\}} = \frac{e^{\ln{\frac{\pi_1}{\pi_2}+\ln{f_1(\bx) - \ln{f_2(\bx)}}}}}{1+e^{\ln{\frac{\pi_1}{\pi_2}+\ln{f_1(\bx) - \ln{f_2(\bx)}}}}}.
\end{align*}
Once again, based on the Whittle log-likelihood given in \eqref{whittleapprox}, we obtain an approximation of the posterior probabilities as $\pi_1(\bx) \approx \sigma(\bx) = e^{\mathcal{L}(\bx)}/(1+e^{\mathcal{L}(\bx)})$, where $\mathcal{L}(\bx)$ is as defined in \eqref{disc_fin}.
We further assume that the class label $Y$ of a given series $\bx$ follows a binomial distribution with parameter $\sigma(\bx)$. Thus, for a given training sample $\rchi_n$, the negative log-likelihood of $(Y_1,\ldots, Y_n)$ is given by
\begin{align*}
    NLL(\widetilde{\bD}_{1:T^\prime}, \Omega^{\bD}_T;\rchi_n) 
    = -\sum\limits_{j=1}^n\left \{\mathbb{I}[{Y_j=1}]\mathcal{L}(\bX_j) + \log{(1+e^{\mathcal{L}(\bX_j)})}\right \}.
\end{align*}
We now add the penalized {\it D-trace loss} given in \eqref{dtraceloss} to the above negative log-likelihood. Define
\begin{align}\label{lassonn}
\ell(\widetilde{\bD}_{1:T^\prime}, \Omega^{\bD}_T;\rchi_n) = NLL(\widetilde{\bD}_{1:T^\prime}, \Omega^{\bD}_T;\rchi_n)+\sum\limits_{\omega_k\in {\Omega}^\bD_T}\left (L(\widetilde{\bD}_k; \hat{\widetilde{\bS}}_{1k}, \hat{\widetilde{\bS}}_{2k})+\lambda\|\widetilde{\bD}_k\|_1\right ),
\end{align}
where $\lambda>0$ is a tuning parameter. We minimize $\ell$ with respect to $\widetilde{\bD}_{1:T^\prime}=\{\widetilde{\bD}_1,\ldots, \widetilde{\bD}_{T^\prime}\}$. In our experiments, we use cross-validation to tune the parameter $\lambda$. We use one-half of the data for training and the other half for validation. We use ADAM for minimizing both \eqref{gic} and \eqref{lassonn}. This choice is useful as ADAM has been shown to work well with a wide variety of convex and non-convex problems and training deep neural networks. It allows us to estimate $\widetilde{\bD}_k$ efficiently using optimization tricks commonly employed while training neural networks that include mini-batching, early stopping, and learning rate scheduling. ADAM uses the gradients $\nabla_{\widetilde{D}_{k}}\ell$ to estimate the difference matrices. We estimate the $(i,j)$-th element  $\widetilde{D}_{k}({i,j})$ as~the~following:
\begin{gather*}
\text{initialize } \widetilde{D}^0_{k}=\{diag(\hat{\widetilde{\bS}}_{2k})\}^{-1}-\{diag(\hat{\widetilde{\bS}}_{1k})\}^{-1},\\
g^{i,j}_{k,iter+1} = \frac{1}{M}\sum_{m=1}^{M}\nabla_{\widetilde{D}^{m}_{k}(i,j)}\ell,\ \ \ 
m^{i,j}_{k,iter+1} = \beta_{1}\times m^{i,j}_{k,iter} + (1-\beta_{1}) \times g^{i,j}_{k,iter+1} \\
v^{i,j}_{k,iter+1} = \beta_{2}\times v^{i,j}_{k,iter} + (1-\beta_{2}) \times g^{i,j}_{k,iter+1}.g^{i,j}_{k,iter+1},\ \ \ 
\hat{m}^{i,j}_{k,iter+1} = \frac{m^{i,j}_{k,iter+1}}{1-B_{1}^{iter+1}},\\
\hat{v}^{i,j}_{k,iter+1} = \frac{v^{i,j}_{k,iter+1}}{1-B_{2}^{iter+1}},\ \ \ 
\Delta D_{k}^{iter+1}(i,j) =  \Delta D_{k}^{iter}(i,j) - \frac{\alpha \times \hat{m}^{i,j}_{k,iter+1} }{\delta + \sqrt{\hat{v}^{i,j}_{k,iter+1}}}, \\
 D_{k}^{iter+1}(i,j) =  D_{k}^{iter}(i,j) - \Delta D_{k}^{iter+1}(i,j).
\end{gather*}
Here, $m$, $k$ refer to the momentum and velocity terms in the standard ADAM optimizer.
\subsection{Screening of relevant Fourier frequencies}\label{freqscreen}
Recall the definition of the sets $\Omega^{0}_T$ and $\Omega^\bD_T$ given in \eqref{nset}. By estimating these sets, we propose a data-adaptive method to screen the Fourier frequencies relevant to the underlying classification problem. Define 
$d_k =\|\widetilde{\bD}_k\|_{\mathrm{F}}$ for $k=1,\ldots, T^\prime$.
It follows from the definition that $d_k$ is exactly zero for all $\omega_k\in \Omega_T^0$ and is strictly positive for all $\omega_k\in \Omega^\bD_T$. This allows us to express the sets $\Omega^{0}_T$ and $\Omega^\bD_T$ as
    $\Omega^{0}_T = \{\omega_k: d_{k}=0$, $\omega_k\in \Omega_T\}$,
and $\Omega^\bD_T = \{\omega_k: d_{k}> 0$, $\omega_k\in \Omega_T\}$.
Thus the frequency screening problem is now reduced to a problem identifying strictly positive $d_k$-s. We now present a data-driven procedure to screen the $d_k$s that are significantly large.

Let $d_{(k)}$ denote the $k$-th minimum among $d_1,\ldots, d_{T^\prime}$. Observe that if we arrange the values  $\{d_k:1\leq k\leq T^\prime\}$ in increasing order of magnitude, then the smallest $T_{0}$ values will correspond to the set $\Omega^{0}_T$ and will all be equal to zero. In other words, we have  
    $0 = d_{(1)}=\cdots = d_{(T_0)}<d_{(T_0+1)}\leq \cdots \leq d_{(T^\prime)}$.
This leads to yet another equivalent representation of $\Omega^\bD_T$:
\begin{align*}
    \Omega^\bD_T = \{d_k: d_k\geq d_{(T_0+1)}\text{ for }1\leq k\leq (T^\prime-1)\}.
\end{align*}
The above definition suggests that to estimate $\Omega^\bD_T$, we only need to estimate $T_0$. Another key observation is that the ratio $r_k = d_{(k+1)}/d_{(k)}<\infty$ for all $(T_0+1)\leq k\leq (T^\prime -1)$  whereas~$r_{T_0}=\infty$. 
Let $\hat{d}_k=\|\hat{\widetilde{\bD}}_{k}\|_{\mathrm{F}}$ and 
consider $\hat{d}_{(1)}< \ldots <\hat{d}_{(T^\prime)}$. Since the underlying distributions are absolutely continuous with respect to the Lebesgue measure, the ratios $\hat{r}_k = \hat{d}_{(k+1)}/\hat{d}_{(k)}$ for $k=1,\ldots,(T^\prime-1)$ are well-defined. Since $r_{T_0} = \infty$, we expect the ratio $\hat{r}_{T_0}$ to take a significantly large value when compared to the entire sequence $\{\hat{r}_k:1\leq k\leq (T^\prime -1)\}$. We define
\begin{align}\label{est_partition}
    \hat{T}_{0}&= \argmax\limits_{1\leq k\leq  (T^\prime-1)}\hat{r}_k,\text{ and }\hat{T}_{\bD} = T^\prime - \hat{T}_0,\ \text{and}\ 
    \hat{\Omega}^\bD_T = \{\omega_k: \hat{d}_{k} > \hat{d}_{(\hat{T}_0)},\ k=1,\cdots, T^\prime \}.
\end{align}
\noindent In practice, one may work with $\hat{T}_0 =\argmax_{T_{min}\leq k\leq  T_{max}}\hat{r}_k$, where $T_{min}$ and $T_{max}$ are user defined constants. In variable screening literature, similar criteria based on ratios of ordered values have been discussed in \cite{MR3514508} and \cite{roy2023exact}. Note that the ordering of $\hat{d}_k$-s immediately gives the relative importance of fundamental frequencies in classification.
\subsection{Classification methods}
Recall the classifier $\delta_W$ and the discriminant $\mathcal{L}(\bz)$ defined in \eqref{class_BW} and \eqref{disc_fin}, respectively. The SDMs in $\mathcal{L}(\bz)$ are estimated by smoothed periodogram estimator. In section \ref{method}, we have developed two different methods to estimate the $\bD_k$-s by minimizing (i) the lasso penalized {\it D-trace loss}, and (ii) the sum of lasso penalized {\it D-trace loss} and the cross-entropy loss. In Section \ref{freqscreen} we estimated the set of relevant Fourier frequencies based on the Frobenius norms of estimated $\bD_k$-s. The prior probabilities $\pi_1$ and $\pi_2$ are estimated by $\hat{\pi}_1 = n_1/(n_1+n_2)$, $\hat{\pi}_2 = n_2/(n_1+n_2)$, respectively. Let the solutions to the problem in \eqref{gic} and \eqref{lassonn} be denoted by $\hat{\widetilde{\bD}}_{1k}$ and $\hat{\widetilde{\bD}}_{2k}$, respectively. Also, let $\hat{\Omega}^{\bD_1}_T$ and $\hat{\Omega}^{\bD_2}_T$ denote the estimated set of relevant frequencies based on $\hat{\widetilde{\bD}}_{1k}$ and $\hat{\widetilde{\bD}}_{2k}$, respectively.  Using this estimates in \eqref{est_partition}, we propose two classification rules:
\begin{gather}\label{disc.prop}
\delta_j(\bz)=
       \mathbb{I}[\hat{\mathcal{L}}_j(\bz)>0] +        2\mathbb{I}[\hat{\mathcal{L}}_j(\bz)\leq 0], \nonumber \\
   \text{where }\hat{\mathcal{L}}_j(\bz)=\ \ln {\frac{\hat{\pi}_1}{\hat{\pi}_2}} + \frac{1}{2}\sum\limits_{\omega_k\in \hat{\Omega}^{\bD_j}_{T}}\left [\tilde{z}_k^\top\hat{\widetilde{\bD}}_{jk}\tilde{z}_k+\ln{|\hat{\widetilde{\bD}}_{jk}\hat{\widetilde{\bS}}_{1k} + \bI|} \right ]\text{ for }j=1,2.
\end{gather}
In the next section, we present the theoretical properties of $\hat{\widetilde{\bD}}_{1k}$-s and $\hat{\Omega}^{\bD_1}_T$, i.e., the estimates based on the convex minimization problem in Section \ref{est_Domega}.
\section{Theoretical properties}\label{thprop}
Let $\tilde{S}_k=\{(i,j):\widetilde{\bD}_k(i,j)\neq 0\}$ be the support of $\widetilde{\bD}_k$, and $\tilde{s}_k=|\tilde{S}_k|$. 
Suppose that $\Gamma(\bA,\bB) = (\bA \bigotimes \bB + \bB \bigotimes \bA)/2$ where $\bigotimes$ denotes the Kronecker product between two $2p\times 2p$ matrices $\bA$ and $\bB$. For any two subsets $P_1$ and $P_2$ of $\{1,\ldots, 2p\}\times \{1,\ldots, 2p\}$, we denote by $\Gamma_{P_1P_2}(\bA,\bB)$ the submatrix of $\Gamma(\bA,\bB)$ with rows and columns indexed by $P_1$ and $P_2$, i.e.,  $\Gamma_{P_1P_2}(\bA,\bB) = \frac{1}{2}(A_{j,l}B_{k,m}+A_{k,m}B_{j,l})_{(j,k)\in P_1, (l,m)\in P_2}$. For notational simplicity, we write $\Gamma_k = \Gamma(\widetilde{\bS}_{1k}, \widetilde{\bS}_{2k})=\left (\Gamma_k(i,j)\right )$. We always assume that $\max_k \max (\|\widetilde{\bS}_{1k}\|_\infty, \|\widetilde{\bS}_{2k}\|_\infty)\leq 2M$ and $\max_k \max (\|\widetilde{\bS}_{1k}\|_{1,\infty}, \|\widetilde{\bS}_{2k}\|_{1,\infty})\leq 2M^*$ for some constants $M$ and $M^*$ independent of $p$ and $T$. We define the following quantities:
\begin{gather}\label{alpha_kappa}
    \alpha_k = 1 - \max\limits_{e\in \tilde{S}^c_k}\|\Gamma_{k, e\tilde{S}_k}(\Gamma_{k, \tilde{S}_k\tilde{S}_k})^{-1}\|_1,    \kappa_{\Gamma_k} =\|(\Gamma_{k, \tilde{S}_k\tilde{S}_k})^{-1}\|_{1,\infty},\text{ for }k=1,\ldots, T^\prime,\nonumber \\
\tilde{s}_{\max} =\max_k \tilde{s}_k,\
    \theta_{\eta}(n, p) = (\eta\ln{p}+\ln{4})/n,\ \rho_\eta(n,p) = 1/(1+C_1\theta^{-\frac{1}{2}}_{\eta}(n, p)), \text{ and }\nonumber\\
    \psi_\eta(n,p) = C_2M^2\left (\theta^\frac{1}{2}_\eta(n,p)+C_3\theta_\eta(n,p)\right )\max_k  \tilde{s}_{k}\kappa_{\Gamma_k }(M^2\max_k  \tilde{s}_{k}\kappa_{\Gamma_k }+1).
\end{gather}
Observe that $\tilde{S}_k=\emptyset$ for $\omega_k\in\Omega^0_T$. Therefore, $\Gamma_{k, \tilde{S}_k\tilde{S}_k}$ is essentially an empty matrix with 0 rows and 0 columns for all $\omega_k\in\Omega^0_T$. 
Similarly, the matrix $\Gamma_{k, e\tilde{S}_k}$ is also empty for all $e\in \tilde{S}^c_k$. Thus, $\|\Gamma_{k, e\tilde{S}_k}(\Gamma_{k, \tilde{S}_k\tilde{S}_k})^{-1}\|_1=\|(\Gamma_{k, \tilde{S}_k\tilde{S}_k})^{-1}\|_{1,\infty}=0$ for all $k$ with $\omega_k\in\Omega^0_T$. Consequently, $\alpha_k=1$ and $\kappa_{\Gamma_k} = 0$ for all $k$ with $\omega_k\in\Omega^0_T$. Consider the following assumptions:
\begin{enumerate}
    \item[A1.] There exists a constant $\eta_1>2$ such that (a) $\min_l\min_k\max_j \bS_{lk}(j,j)>\sqrt{2}M\theta_{\eta_1}(n, p)$, (b) $\max_{k}\tilde{s}_k\kappa_{\Gamma_k}<o\left (\theta^{-\frac{1}{2}}_{\eta_1}(n,p)\right )$, and (c) $\min_k\alpha_k>4\max\{\rho_{\eta_1}(n,p), \psi_{\eta_1}(n,p)\}$.
\end{enumerate}
We consider $p$ to be increasing with $n$ and $T$. It follows from the definition of $\theta_\eta(n,p)$ that if there exists a $0<a<1$, such that $\log{p} = n^a$, then $\theta_\eta(n,p)\to 0$ as $n\to\infty$. Therefore, assumption A1.a allows the variance of $X(\omega_k)$ to decrease for $l=1,2$ and $\omega_k\in\Omega^\bD_T$, but at a slower rate than $\theta_\eta(n,p)$. Similarly, A1.b implies that the quantity $\max_{k}\tilde{s}_k\kappa_{\Gamma_k}$ cannot grow arbitrarily. Its rate of growth is governed by $\theta_\eta(n,p)$. Both $\rho_\eta$ and $\psi_\eta$ are decreasing sequences in $n$. A1.c implies that $\max_{e\in \tilde{S}^c_k}\|\Gamma_{k, e\tilde{S}_k}(\Gamma_{k, \tilde{S}_k\tilde{S}_k})^{-1}\|_1<1$ for all $k$ which is the \textit{irrepresentability condition} assumed by \citet{yuan2017differential}. It implies that the elements of $\tilde{S}_k$ and  $\tilde{S}^c_k$ are weakly correlated for all $\omega_k\in\Omega^\bD_T$.
\subsection{Consistent estimation of $\widetilde{\bD}_k$}
We now present the theorem on the convergence of the proposed estimator $\hat{\widetilde{\bD}}_k$ for $k\in\{1,\ldots, T^\prime\}$. Consider the tuning parameter 
\begin{align}\label{lamb_def}
    \lambda_{\eta k} = \max [2MG_{1k}/A_k, \{128(\eta\ln{p} + \ln{4})\}^\frac{1}{2}\widetilde{M}_k G_{2k} + MG_{1k} \widetilde{M}_k]\{128(\eta\ln{p} + \ln{4})\}^\frac{1}{2}
\end{align}
for some $\eta>2$ while minimizing the objective function in \eqref{lassopen}. Here $G_{1k}, G_{2k}, A_k$ and $\widetilde{M}_k$ are constants depending on the quantities in \eqref{alpha_kappa}. Definitions of these constants are detailed in the Appendix. We also use $\bar{\sigma}_k$ in the next theorem which is not defined here due to space constraints.

\begin{theorem}\label{thm1}
    If the assumption A1 is satisfied for $\eta_1>2$, $\min\{n_1,n_2\}> C\max_k\bar{\sigma}_k^{-2}(\eta_1\ln{p} + \ln{4})$ for some $C>0$ and $\widetilde{\bD}_k$ is estimated using $\lambda_{\eta_1 k}$ in \eqref{gic}, then
    \begin{enumerate}
        \item $\mathrm{P}\left [\hat{\tilde{S}}_k\subseteq \tilde{S}_k \text{ for all } \omega_k\in\Omega_T\right ]>1-\frac{2}{p^{\eta_1-2}}$,    
        \item    $\mathrm{P}\left [\max\limits_{\omega_k\in\Omega_T}\|\hat{\widetilde{\bD}}_k - \widetilde{\bD}_k\|_\infty\leq \max_k M_{G_k}\theta_{\eta_1}^\frac{1}{2}(n,p)\right ]>1-\frac{T}{p^{\eta_1-2}},\text{ and}$
        \item $\mathrm{P}\left [\max\limits_{\omega_k\in\Omega_T}\|\hat{\widetilde{\bD}}_k - \widetilde{\bD}_k\|_\mathrm{F}\leq \max_k M_{G_k}\tilde{s}_k^\frac{1}{2}\theta_{\eta_1}^\frac{1}{2}(n,p) \right ]> 1-\frac{T}{p^{\eta_1-2}}$.
    \end{enumerate}
\end{theorem}
Recall that $\tilde{S}_k=\emptyset$ for all $\omega_k\in \Omega^\bD_T$. Theorem 1.1 implies that the estimated support of $\widetilde{\bD}_k$ lies in the true support for all $\omega_k\in\Omega^\bD_T$ and the estimated support is an empty set for all $\omega_k\in\Omega^0_T$, with high probability. Theorem 1.2-1.3 shows the rates of convergence of the estimated $\widetilde{\bD}_k$ for all $\omega_k$.
\subsection{Consistent screening of frequencies}
    \begin{enumerate}
        \item[A2.] There exists $\eta_2\geq \eta$ such that $\min_{j,l\in\tilde{S}_k}|\widetilde{\bD}_k(j,l)|> 2M_{G_k}\theta_{\eta_2}^\frac{1}{2}(n,p)\ \text{ for all } \omega_k\in \Omega ^{\bD}_T$.
    \end{enumerate}
    Define $q_\eta(T, n,p)= 2\max_k M_{G_k}\tilde{s}_k^\frac{1}{2}\theta^\frac{1}{2}_\eta(n,p)$. It readily follows from assumption A2 that
    \begin{align*}
        &\ d_k >2M_{G_k}\tilde{s}^\frac{1}{2}_k\theta_{\eta_1}^\frac{1}{2}(n,p)=q_{\eta_1}(T, n,p)\ \text{ for all } \omega_k\in \Omega ^{\bD}_T,\ \   \text{i.e., } d_{(T_0+1)}> 2q_{\eta_1}(T, n,p).
    \end{align*}
    \begin{enumerate}
        \item[A3.] $\max_{(T_0+1)\leq k\leq T^\prime-1} r^{d}_k <O(d_{(T_0+1)}/q_{\eta_1}(T, n,p))$.
    \end{enumerate}

Assumption A3 implies that the differences between inverse SDMs cannot increase arbitrarily. 
\begin{theorem}\label{thm4}
    If the assumptions A1-A3 are satisfied with $\eta_1,\eta_2>2$, then 
    \[\mathrm{P}\left[ \Omega^\bD_T\subseteq \hat{\Omega}^\bD_T\right ]> 1 - \frac{T}{p^{\eta-2}}\text{ for all }\eta > \max\{\eta_1,\eta_2\}.\]
\end{theorem}
Theorem \ref{thm4} shows that the estimated set of frequencies contains the true set with high probability. Therefore, if $T = o(p^{\eta -2})$, then the proposed screening method possesses {\it sure screening property}. 

\section{Simulation study}\label{simstud}
In this section, we compare the performance of the proposed classifiers, \(\delta_1\) and \(\delta_2\), with several popular high-dimensional classifiers: linear discriminant classifier (LDA), quadratic discriminant classifier (SQDA), 1-nearest neighbor (1NN) with dynamic time warping \citep[DTW,][]{JSSv099i09}, and the state-of-the-art MiniROCKET. LDA and QDA were implemented using scikit-learn package in python \cite{scikit-learn}. SVD solver was used for LDA and $\lambda$ $\in$ $[0,1]$. Minirocket was implemented with default 10,000 kernels using sktime \cite{loning2019sktime} package followed by Ridge Classifier with cross validation to tune regularization parameter $\lambda$ $\in$ $[-3,3]$.For ADAM we use $lr = 1e-3$ ,$\lambda$ $\in$ $[1,1e-10]$.Hyper-parameters for all models were tuned using 2 fold CV. All experiments were performed on CPU on Intel Core i5. We also evaluate the two estimation procedures from Sections \ref{est_Domega} and \ref{estnn}, the accuracy of the frequency screening method.
Let $\tilde{X}(\omega_k)\mid Y=l\sim N_{2p}(\0, \widetilde{\bS}_{lk})$ for $\omega_k=0,1/T,\ldots, [(T-1/2)]/T$, and $X(t)$ be the inverse Fourier transformation based on the DFTs ${X}(\omega_k)=\tilde{X}_{1:p}(\omega_k)+\mathrm{i}\tilde{X}_{p+(1:p)}(\omega_k)$. The matrix $\widetilde{\bTheta}_{lk}=\widetilde{\bS}^{-1}_{lk}$ has the structure as introduced in \ref{method}, i.e., column and row augmentation of the following symmetric and skew-symmetric matrices: 
\begin{gather*}
 \widetilde{\bTheta}^\mathcal{R}_{1k}(i,j) = 0.5^{|i-j|},\ \    \widetilde{\bTheta}^\mathcal{I}_{1k}(i,j) = 0.5^{|i-j|} * \{-1\}^{\mathbb{I}[j>i]}-\mathbb{I}[j=i],\\
 {\bTheta}_{2k}(i,j)={\bTheta}_{1k}(i,j)  \{-1\}^{\mathbb{I}[|j-i|=1]}\text{ with }  \bD_{k}(i,j) = \mathbb{I}[|i-j|=1],
\end{gather*}
if $\omega_k\in\{1/T, \ldots, 20/T\}=\Omega^\bD_T$. For the remaining frequencies, \(\bTheta_{1k} = \bTheta_{2k}\). Thus, differences between the two populations occur only at frequencies in \(\Omega^\bD_T\). Note that \(\bD_{k}\) is sparse on \(\Omega^\bD_T\) and null on \(\Omega^0_T\). We consider two examples with these settings. In \textbf{Example 1}, we have \(n_1 = n_2 = 100\), \(p = 32\), and \(T = 200\). In \textbf{Example 2}, we increase the dimension to \(p = 64\), keeping the other parameters the same. The test set consists of 200 observations, equally divided between both classes.
\begin{table}[H]
\renewcommand{\arraystretch}{1.35}
    \small
    \centering
    \caption{Comparison of estimated misclassification probability of the proposed classifiers with traditional and benchmark methods based on twenty simulation runs (standard errors in italics) }
    \begin{tabular}{|c|c|c|c|c|c|c|c|c|}
    \hline
         Ex & $(n,p,T)$ & LDA & QDA & 1NN & 1NN & Mini  & $\delta_1$ & $\delta_2$\\
         & & & &DTW-$l_1$ & DTW-$l_2$ &ROCKET & &\\
         \hline
         1 & $(200,32,200)$ & 0.488 & 0.467  & 0.481 & 0.502 & 0.116 & 0.000 &0.022 \\
         & & {\it 0.027}& {\it 0.016}  & {\it 0.001} & {\it 0.005}  & {\it 0.021} & {\it 0.000}&{\it 0.007}\\
         \hline
         2 & $(200,64,200)$ &0.496&0.507 & 0.501& 0.503 &0.044 & 0.000 & 0.000\\
          & &{\it 0.046} & {\it 0.015} &{\it 0.046}& {\it 0.005 }& 0.021&{\it 0.000} & {\it 0.000}\\
         \hline
    \end{tabular}
    \label{tab1Sim}
\end{table}
Table \ref{tab1Sim} shows that the traditional classifier LDA fails, since the populations have no difference in their locations. Both QDA and 1NN-DTW fail due to the \textit{curse of dimensionality}. The state-of-the-art MiniROCKET yields high accuracy, but it does not have interpretability. The proposed classifiers not only outperform the rest, but they also extract meaningful features from the series. The true positive, true negative, and true discovery rates in the recovery of the support $\cup_{\omega_k\in \Omega^\bD_T}\tilde{S}_k$ and of the set $\Omega^\bD_T$ are given by  $(\text{TP}_\bD, \text{TN}_\bD, \text{TD}_\bD)$ and $( \text{TP}_\omega, \text{TN}_\omega, \text{TD}_\omega)$, respectively. Table \ref{screenaccur} shows that the screening method has retained the relevant $\omega_k$ with high accuracy in both examples. The TP index in the estimated $\tilde{S}_k$ is above $91\%$. Using Dtrace and the $\ell$ loss yield similar accuracy.
\begin{gather*}
\text{TP}_\bD=\frac{\sum_{k,i,j}|\tilde{S}_k|\mathbb{I}[\hat{\widetilde{\bD}}_k(i,j)\neq 0,{\widetilde{\bD}}_k(i,j)\neq 0]}{\sum_{k=1}^{T^\prime}|\tilde{S}_k|},\ \text{TN}_\bD = \frac{\sum_{k,i,j}|\tilde{S}^c_k|\mathbb{I}[\hat{\widetilde{\bD}}_k(i,j)= 0,{\widetilde{\bD}}_k(i,j)= 0]}{\sum_{k=1}^{T^\prime}|\tilde{S}^c_k|},\\
    \text{TD}_\bD = \frac{\sum_{k,i,j}|\hat{\tilde{S}}_k|\mathbb{I}[\hat{\widetilde{\bD}}_k(i,j)\neq 0,{\widetilde{\bD}}_k(i,j)\neq 0]}{\sum_{k=1}^{T^\prime}|\hat{\tilde{S}}_k|},
\end{gather*}
\begin{gather*}
    \text{TP}_\omega=\frac{|\hat{\Omega}^\bD_T\cap \Omega^\bD_T|}{|\Omega^\bD_T|},\hspace{0.5cm} \text{TN}_\omega=\frac{|\hat{\Omega}^0_T\cap \Omega^0_T|}{|\Omega^0_T|} ,\hspace{0.5cm} \text{TD}_\omega=\frac{|\hat{\Omega}^\bD_T\cap \Omega^\bD_T|}{|\hat{\Omega}^\bD_T|}.
\end{gather*}
\begin{table}[H]
\renewcommand{\arraystretch}{1.35}
\small
\caption{Performance of the frequency screening method and recovery of support of $\bD_k$-s in simulated examples based on five iterations (standard errors in italics).}
    \centering
    \begin{tabular}{|c|c|c|c|c|c|c|c|c|c|c|c|c|c|}
    \hline
    &\multicolumn{6}{c|}{Frequency screening}&\multicolumn{6}{c|}{Support recovery }\\
    \hline
       Ex  &\multicolumn{3}{c|}{Dtrace}&\multicolumn{3}{c|}{Dtrace + NLL}&\multicolumn{3}{c|}{Dtrace}&\multicolumn{3}{c|}{Dtrace + NLL}\\
       \hline
       & TP$_\omega$ & TN$_\omega$ & TD$_\omega$& TP$_\bD$ & TN$_\bD$ & TD$_\bD$& TP$_\omega$ & TN$_\omega$ & TD$_\omega$& TP$_\bD$ & TN$_\bD$ & TD$_\bD$\\
         \hline
         1 &1.000 & 0.987 & 0.952& 1.000 & 0.987 & 0.952&0.911 &0.402 &0.089& 0.913 & 0.401 & 0.089\\
         &{\it 0.000} & {\it 0.000}& {\it 0.000}&{\it 0.000} & {\it 0.000}& {\it 0.000}&  {\it 0.002} & {\it 0.001} & {\it 0.000} &{\it 0.031} & {\it 0.031}& {\it 0.000}\\
         \hline
         2 & 0.950 & 0.988 & 0.905&0.950 & 0.988 & 0.905&0.911 &0.407 &0.046& 0.913 &0.407 &0.047\\
         &{\it 0.000} & {\it 0.000}& {\it 0.000}&{\it 0.000} & {\it 0.000}& {\it 0.000}&{\it 0.002}&  {\it 0.000} & {\it 0.000} &{\it 0.002}&  {\it 0.001} & {\it 0.000}\\
         \hline
    \end{tabular}
    \label{screenaccur}
\end{table}
\begin{figure}[H]
    \centering
         \includegraphics[scale = 0.485]{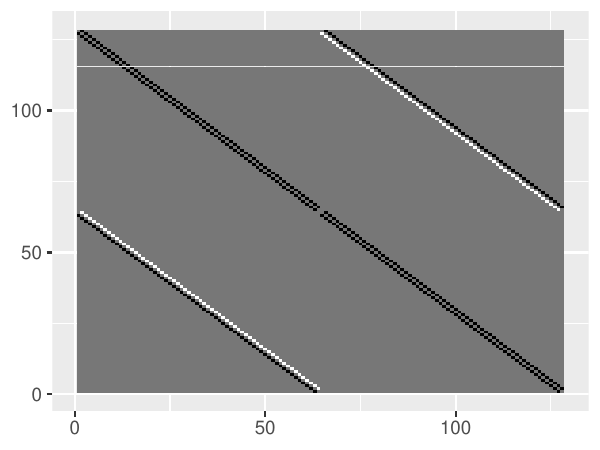}\ \includegraphics[scale = 0.485]{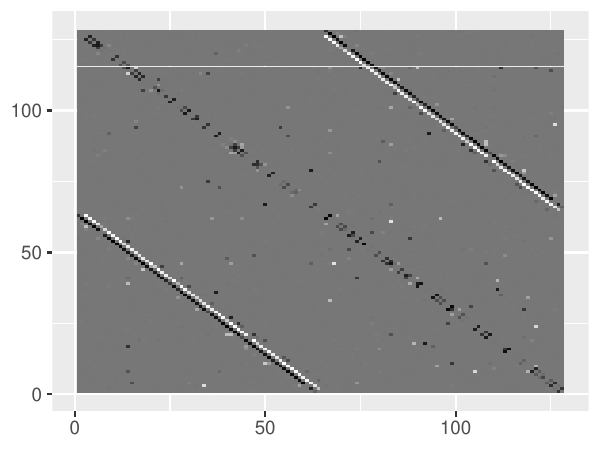}\  \includegraphics[width = 0.35\textwidth]{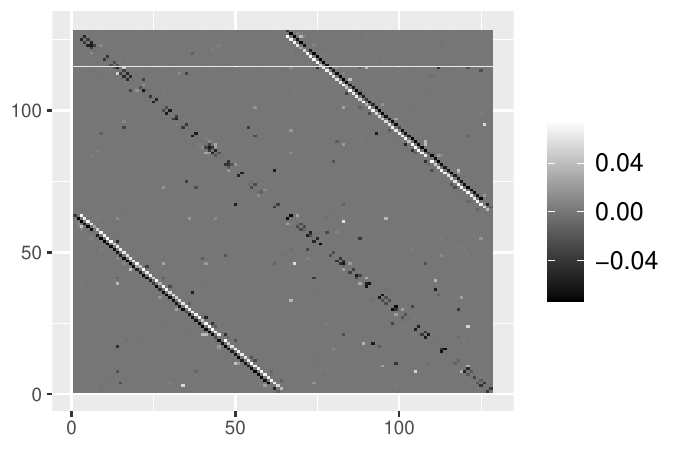}
    \caption{The left panel shows shows the true sparse structure of $\widetilde{\bD}_1$ in Example 2. The middle and the right panels show the structures estimated using Dtrace and $\ell$, respectively (based on the first iteration).}
    \label{sparsestrucEx2}
\end{figure}
\section{Real data analysis}\label{real}
We evaluate our proposed methodology on the EEG drowsy-alert dataset \citep[][]{cao2019multi}, acquired using a virtual reality driving simulator. This dataset includes EEG data from 27 subjects (aged 22-28) recorded at 32 channels at 500Hz, with a total of 1872 epochs of 3.2 seconds from 10 subjects, down-sampled to 128Hz. Baseline alertness was defined as the 5th percentile of local reaction time (RT) after sudden events. Trials with RT lower than 1.5 times the baseline were labeled `alert/normal,' and those higher than 2.5 times were labeled `drowsy.' Moderate RT trials were excluded. Each epoch is a 30-dimensional time series of length 384. We randomly selected $50\%$ of the trials for the training set, maintaining equal class proportions. The remaining data was split equally into validation and test sets.
\begin{figure}[H]
    \centering
    \begin{subfigure}[t]{0.5\textwidth}
        \centering
        \caption{$\sum_{\omega_k\in \hat{\Omega}_T\cap \alpha-\text{band}}|\bD(\omega_k)|$}
        \includegraphics[height = 2.10in, width = 0.99\textwidth]{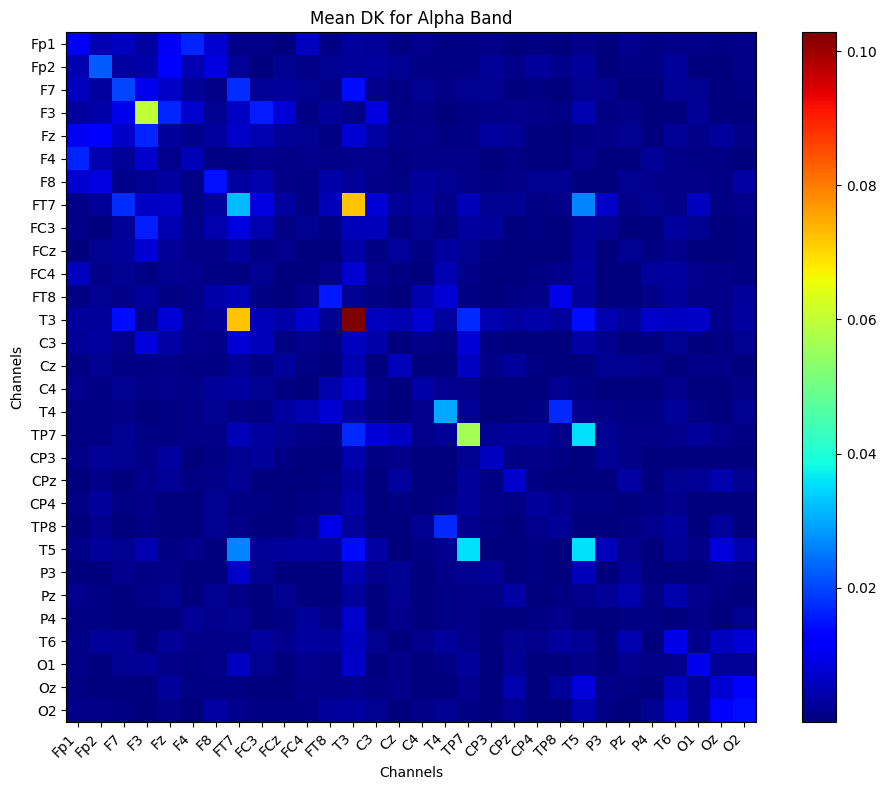}
    \end{subfigure}%
    ~ 
    \begin{subfigure}[t]{0.5\textwidth}
        \centering
        \caption{$\sum_{\omega_k\in \hat{\Omega}_T\cap \beta-\text{band}}|\bD(\omega_k)|$}
        \includegraphics[height=2.10in, width = 0.99\textwidth]{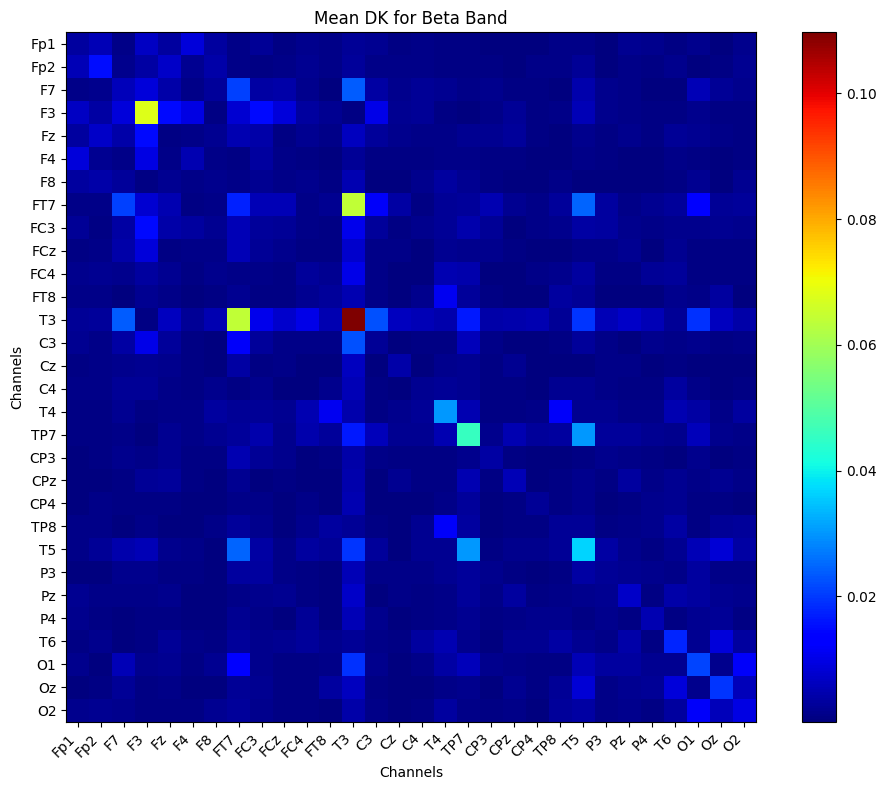}
    \end{subfigure}\\
    \begin{subfigure}[t]{0.5\textwidth}
        \centering
        \caption{$\sum_{\omega_k\in \hat{\Omega}_T\cap \gamma-\text{band}}|\bD(\omega_k)|$}
        \includegraphics[height=2.10in, width = 0.99\textwidth]{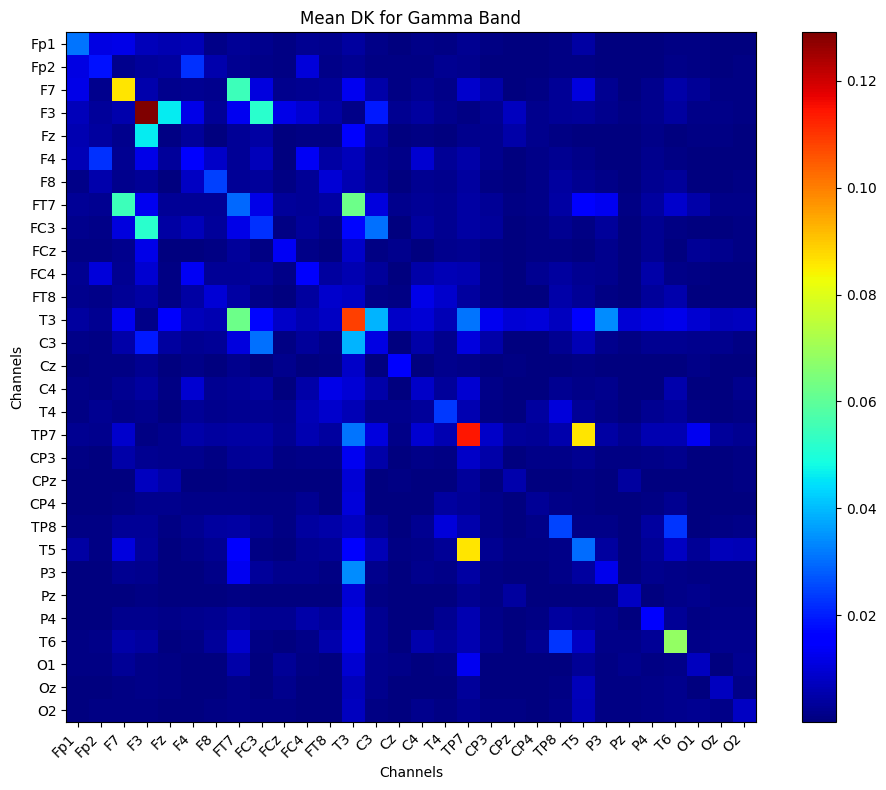}
    \end{subfigure}%
    ~ 
    \begin{subfigure}[t]{0.5\textwidth}
        \centering
        \caption{$\sum_{\omega_k\in \hat{\Omega}_T\cap \delta-\text{band}}|\bD(\omega_k)|$}
        \includegraphics[height=2.10in, width = 0.99\textwidth]{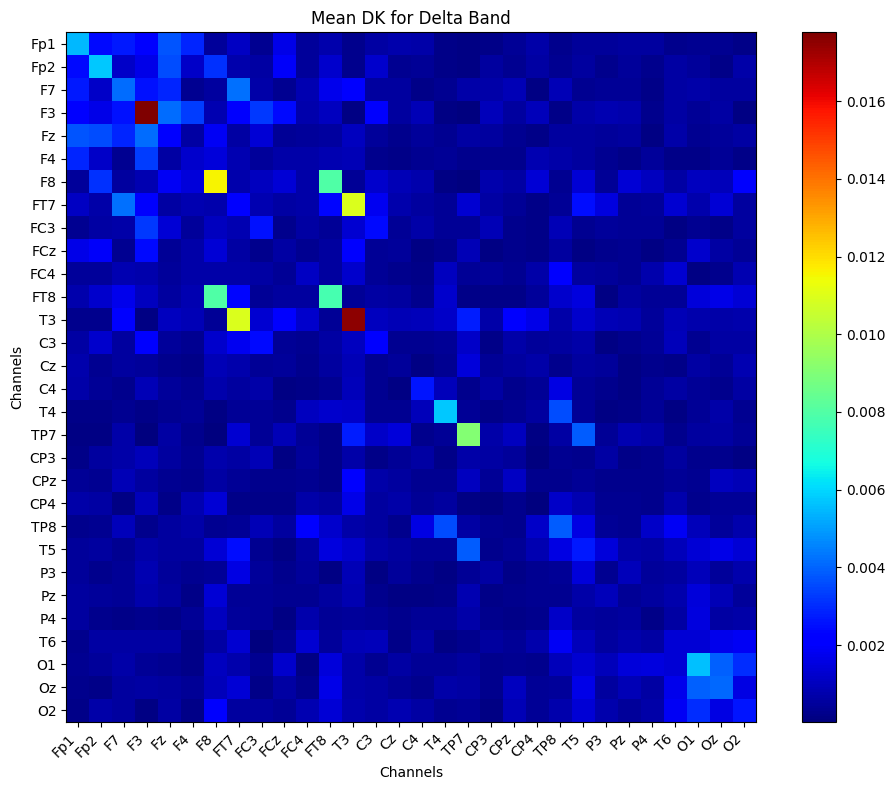}
    \end{subfigure}\\
    \begin{subfigure}[t]{0.5\textwidth}
        \centering
        \caption{$\sum_{\omega_k\in \hat{\Omega}_T\cap \theta-\text{band}}|\bD(\omega_k)|$}
        \includegraphics[height=2.10in, width = 0.99\textwidth]{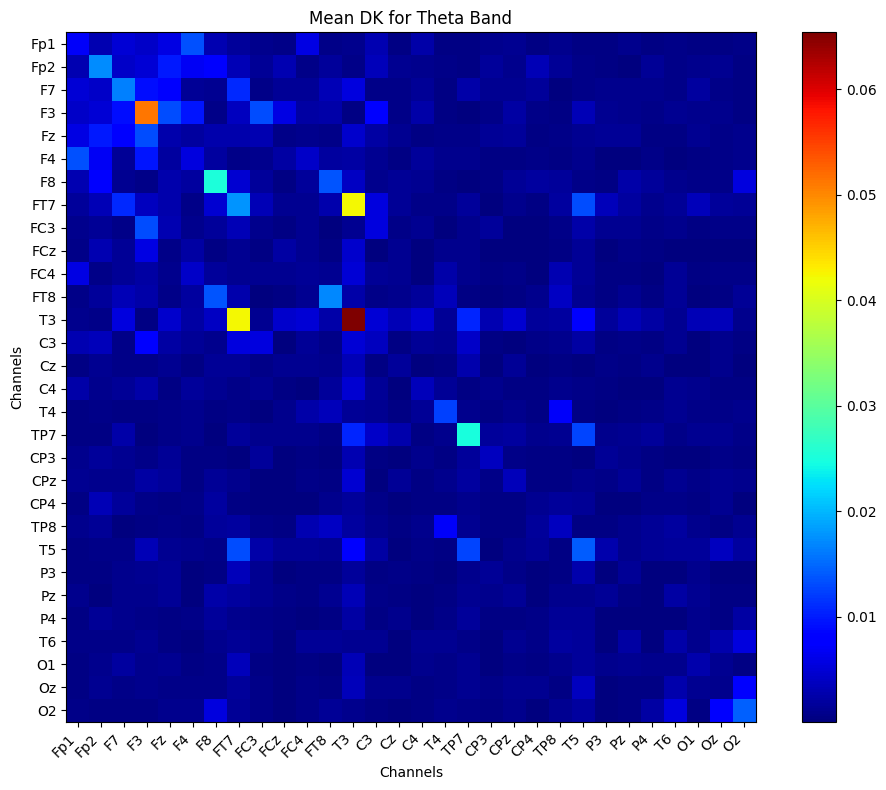}
    \end{subfigure}%
    ~ 
    \begin{subfigure}[t]{0.5\textwidth}
        \centering
        \caption{$|\boldsymbol{\Gamma}_1(0) - \boldsymbol{\Gamma}_2(0)|$}
        \includegraphics[height=2.10in, width = 0.99\textwidth]{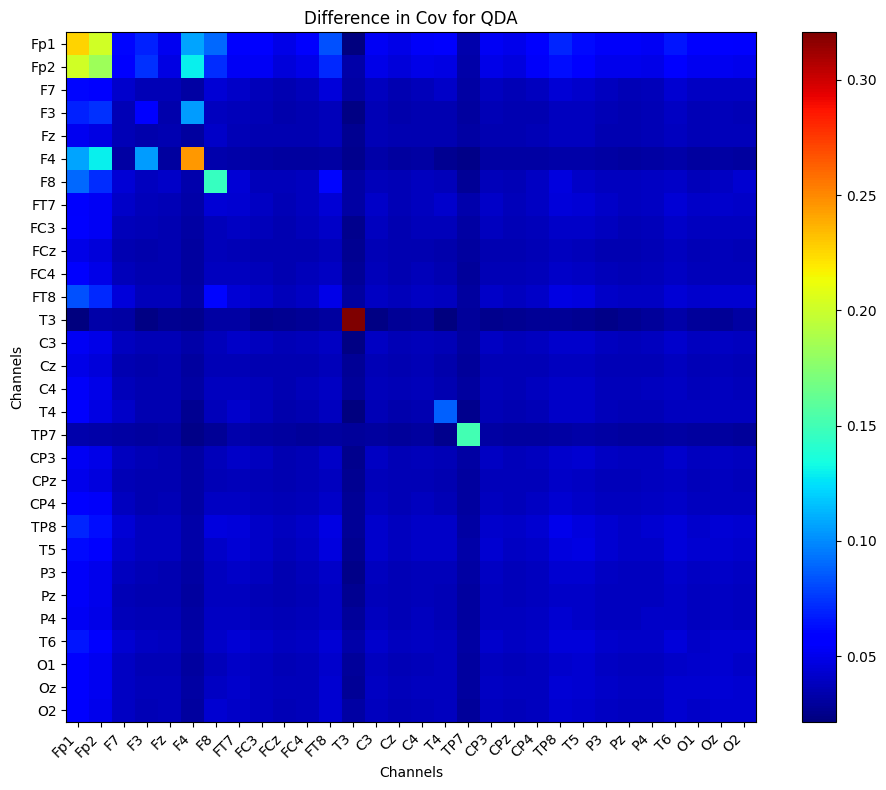}
    \end{subfigure}
    \caption{(e)-(d)--Clockwise from bottom-left -- the interactions between channels for frequency bands $\theta,\gamma,\alpha,\beta$ and $\delta$ are displayed using heatmaps. (f) The difference between auto-covariance matrices $\boldsymbol{\Gamma}_1(0)$ and $\boldsymbol{\Gamma}_2(0)$. $|\bM|$ denotes elements wise absolute value of matrix $\bM$.}
    \label{ad_plot}
\end{figure}
We observed no significant difference in the mean of the two classes. This explains the poor performance of the linear classifier LDA. In Figure \ref{ad_plot}(f), we plot the difference between $30 \times 30$ auro-covariance matrices of class 1 and 2 at lag 0. It suggests the presence of discriminatory information in the interaction of channels, particularly in frontal and temporal regions. We observed a similar pattern of differences (but weaker in magnitude) in auto-covariance matrices at higher lags (omitted from the global response due to space constraints). This difference in connectivity patterns explains the performance of QDA. The proposed method extracts these to better explain changes in interaction patterns between channels (see Figure \ref{ad_plot}(a)-(e)). For example, we observe significant differences in the pattern of interaction between the Frontal lobe channels for the Gamma band which is associated with an elevated state of vigilance and cognitive activity.

\begin{table}
\renewcommand{\arraystretch}{1.35}
\small
    \centering
    \caption{Comparison of estimated misclassification probability of the proposed classifiers with traditional and benchmark methods based on twenty iterations (standard errors in italics) }
    \begin{tabular}{|c|c|c|c|c|c|c|c|c|c|}
    \hline
          $(n,p,T)$ & LDA & QDA & 1NN & 1NN &Mini & $\delta_1$ & $\delta_2$\\
          &  &  &DTW-$l_1$ & DTW-$l_2$ &ROCKET & &\\
         \hline
          $(1872,30,384)$ &0.492 & 0.296 & 0.478 & 0.473 & 0.155 & 0.269 & 0.258\\
          &{\it 0.008} &{\it 0.036} & {\it 0.002} & {\it 0.001} & {\it 0.012} & {\it 0.020} & {\it 0.026}\\
         \hline
    \end{tabular}
    \label{tab1Real}
\end{table}
\section{Concluding remarks}\label{limit}
In this article, we present a statistical method to classify high-dimensional time series in the spectral domain, with an emphasis on parameter interpretability. We establish the consistency of the proposed estimators of the model parameters. The classifier assumes the time series are Gaussian and stationary. Future research will extend this framework to accommodate non-Gaussian distributions, such as elliptical densities. The consistency of our proposed classifier hinges on the difference between Gaussian and Whittle likelihoods. Investigating this difference in high dimensions remains an open problem and is beyond the scope of this article. Additionally, employing a hierarchical group lasso penalty in the loss function could better manage varying degrees of sparsity. Another research direction is developing models for block-stationary or locally stationary time series.

{
\bibliographystyle{apalike}
\bibliography{reference}
}

\newpage
\appendix
\section{Appendix}\label{appen}
\subsection{Notations} \label{Nota_C}

Throughout this article, we have used the following definitions to present the mathematical results and in related discussions.
\begin{enumerate}
	\item $a_n=o(b_n)$ implies that for every $\epsilon>0$ there exists an $N\in\mathbb{N}$ such that $|a_n/b_n|<\epsilon$ for all $n\geq N$.
	\item $a_n=O(b_n)$ implies that there exist $M>0$ and $N\in\mathbb{N}$ such that $|a_n/b_n|<M$ for all $n\geq N$.
	\item $X_n=o_{\rm P}(a_n)$ implies that the sequence of random variables $X_n/a_n$ converges to 0 in probability as $n\to\infty$. 
	\item $X_n=O_{\rm P}(a_n)$ implies that the sequence of random variables $X_n/a_n$ is stochastically bounded, i.e., for every $\epsilon>0$ there exist a finite $M>0$ and $N\in\mathbb{N}$ such that ${\rm P}[|X_n/a_n|\ge M]<\epsilon$ for all $n\geq N$.
\end{enumerate}

\subsection{Proofs and Mathematical Details} \label{Appendix_A}
\allowdisplaybreaks
\begin{lemma}\label{normrelations}
    $\|\hat{\bD}_k - \bD_k\|_\infty\leq \sqrt{2}\|\hat{\widetilde{\bD}}_k - \widetilde{\bD}_k\|_\infty$ and $\|\hat{\bD}_k - \bD_k\|_\mathrm{F}= \|\hat{\widetilde{\bD}}_k - \widetilde{\bD}_k\|_\mathrm{F}/\sqrt{2}$ for all $k=1,\ldots, T^\prime$.
\end{lemma}
\begin{proof}
    Note that
\begin{align*}
    \|\hat{\bD}_k - \bD_k\|_\infty =&\ \max \limits_{j,l}\left (\left |\hat{\bD}^\mathcal{R}_k(j,l) + \mathrm{i}\hat{\bD}^\mathcal{I}_k(j,l) - \bD^\mathcal{R}_k(j,l) - \mathrm{i}\bD^\mathcal{I}_k(j,l)\right |\right )\\
    =&\ \max \limits_{j,l}\left (\left (\hat{\bD}^\mathcal{R}_k(j,l) - \bD^\mathcal{R}_k(j,l)\right )^2 + \left (\hat{\bD}^\mathcal{I}_k(j,l) - \bD^\mathcal{I}_k(j,l)\right )^2\right )^\frac{1}{2}\\
    \leq &\ \max \limits_{j,l}\left (2\max \left \{\left (\hat{\bD}^\mathcal{R}_k(j,l) - \bD^\mathcal{R}_k(j,l)\right )^2 , \left (\hat{\bD}^\mathcal{I}_k(j,l) - \bD^\mathcal{I}_k(j,l)\right )^2\right \}\right )^\frac{1}{2}\\
    \leq &\ \max \limits_{j,l}\left (\sqrt{2}\max \left \{\left |\hat{\bD}^\mathcal{R}_k(j,l) - \bD^\mathcal{R}_k(j,l)\right | , \left |\hat{\bD}^\mathcal{I}_k(j,l) - \bD^\mathcal{I}_k(j,l)\right |\right \}\right )\\
    = &\ \sqrt{2}\max \left \{\max \limits_{j,l}\left |\hat{\bD}^\mathcal{R}_k(j,l) - \bD^\mathcal{R}_k(j,l)\right | , \max \limits_{j,l}\left |\hat{\bD}^\mathcal{I}_k(j,l) - \bD^\mathcal{I}_k(j,l)\right |\right \}\\
    =&\ \sqrt{2} \max\limits_{j,l} |\hat{\widetilde{\bD}}_k(j,l) - \widetilde{\bD}_k(j,l)| = \sqrt{2} \|\hat{\widetilde{\bD}}_k - \widetilde{\bD}_k\|_\infty.
\end{align*}
Similarly,
\begin{align*}
    \|\hat{\bD}_k - \bD_k\|_\mathrm{F} =&\ \left (\sum \limits_{j,l}\left |\hat{\bD}^\mathcal{R}_k(j,l) + \mathrm{i}\hat{\bD}^\mathcal{I}_k(j,l) - \bD^\mathcal{R}_k(j,l) - \mathrm{i}\bD^\mathcal{I}_k(j,l)\right |^2\right )^\frac{1}{2}\\
    =&\ \left (\sum \limits_{j,l}\left \{\left (\hat{\bD}^\mathcal{R}_k(j,l) - \bD^\mathcal{R}_k(j,l)\right )^2 + \left (\hat{\bD}^\mathcal{I}_k(j,l) - \bD^\mathcal{I}_k(j,l)\right )^2\right \}\right )^\frac{1}{2}\\
    = &\ \left (\frac{1}{2}\sum \limits_{j,l}\left (\hat{\widetilde{\bD}}_k(j,l) - \widetilde{\bD}_k(j,l)\right )^2\right )^\frac{1}{2}= \frac{1}{\sqrt{2}}\|\hat{\widetilde{\bD}}_k - \widetilde{\bD}_k\|_\mathrm{F}.
\end{align*}
This completes the proof. 
\end{proof}
Consider the following quantities:
\begin{align}\label{quantities}
    \hat{\Gamma}_k =& \Gamma(\hat{\widetilde{\bS}}_{1k},\hat{\widetilde{\bS}}_{2k})\ \kappa_{\Gamma_k} =\ \|(\Gamma_{k, \tilde{S}_k\tilde{S}_k})^{-1}\|_{1,\infty},\ \kappa_{\Gamma^\top_k} = \|(\Gamma^\top_{k, \tilde{S}_k\tilde{S}_k})^{-1}\|_{1,\infty},\text{ and }\tilde{\sigma}^2_{kl} = \max_j\ (5\widetilde{\bS}_{lk}(j,j)/2)^2.\nonumber\\
    G_{1k} =& \frac{\tilde{\sigma}_{k\bX}}{\sqrt{n_1}} + \frac{\tilde{\sigma}_{k\bY}}{\sqrt{n_2}}, \ G_{2k} = \frac{\tilde{\sigma}_{k\bX}}{\sqrt{n_1}}\frac{\tilde{\sigma}_{k\bY}}{\sqrt{n_2}},\nonumber\\
    A_k = &\ M\alpha_k/(4-\alpha_k),\ C_G = 3200M^2,\widetilde{M}_k = 24 M\tilde{s}_{k} (2M^2\tilde{s}_{k}\kappa^2_{\Gamma_k} +\kappa_{\Gamma_k})/\alpha_k,\nonumber\\
    \bar{\sigma}_k =&\ \min \Bigg \{-M + \left [M^2 + (6\tilde{s}_{k}\kappa_{\Gamma_k} )^{-1}\right ]^\frac{1}{2}, -M + \left [M^2 + \frac{\alpha_k}{24\tilde{s}_{k}(2\tilde{s}_{k}M^2\kappa^2_{\Gamma_k }+\kappa_{\Gamma_k})}\right ]^\frac{1}{2},\nonumber \\
    &\ \ \ \hspace{1in} A_k, 16\tilde{\sigma}_{k\bX},16\tilde{\sigma}_{k\bY}\Bigg \},\nonumber \\
   M_{G_k} =&\ 240\sqrt{2}M^2\tilde{s}_{k}\kappa^2_{\Gamma_k} (2M + A_k) + 40\sqrt{2}M\left \{ \kappa_{\Gamma_k} + 3\tilde{s}_{k} \kappa^2_{\Gamma_k}A_k (2M + A_k)\right \}\nonumber \\
    &\ \times \left [2 + \max \left \{ 24M\tilde{s}_{k}(2M^2\tilde{s}_{k}\kappa^2_{\Gamma_k} + \kappa_{\Gamma_k})(2M + A_k)/\alpha_k , 4(4-\alpha_k)/\alpha_k\right \} \right ].
\end{align}

\noindent It is clear from the above definitions that for all $k$ with $\omega_k\in \Omega_T\backslash\Omega^\bD_T$,
\begin{gather}\label{const-noise}
\kappa_{\Gamma_k}=\kappa_{\Gamma^\top_k}=0,\ \widetilde{M}_k=0,\nonumber\\
\bar{\sigma}_k =\min \left \{A_k, 16\tilde{\sigma}_{k\bX},16\tilde{\sigma}_{k\bY}\right \},\text{ and }M_{G_k}=0.
\end{gather}
For the consistency results to hold we will be needing \[\textcolor{black}{\min\{n_1,n_2\}> 3200M^2\max_k\bar{\sigma}_k^{-2}(\eta\ln{p} + \ln{4})\equiv \min_k\bar{\sigma}_k>40\sqrt{2}M\theta^\frac{1}{2}_\eta(n,p)}.\]
Observe that a sufficient condition for $\min_k\bar{\sigma}_k>40\sqrt{2}M\theta^\frac{1}{2}_\eta(n,p)$ to hold is the following:
\begin{align*}
    \min \Bigg \{&\min_k \left (-M + \left [M^2 + (6\tilde{s}_{k}\kappa_{\Gamma_k} )^{-1}\right ]^\frac{1}{2}\right ),\ \min_k \left (-M + \left [M^2 + \frac{\alpha_k}{24\tilde{s}_{k}(2\tilde{s}_{k}M^2\kappa^2_{\Gamma_k }+\kappa_{\Gamma_k})}\right ]^\frac{1}{2}\right ),\\
    &\min_k A_k,\ \min_k 16\tilde{\sigma}_{k\bX},\ \min_k 16\tilde{\sigma}_{k\bY}\Bigg \}>40\sqrt{2}M\theta^\frac{1}{2}_\eta(n,p).
\end{align*}
This sufficient condition reduces to the following set of inequalities:
\begin{align}\label{suff_ineqs}
    &\min_k \left [M^2 + (6\tilde{s}_{k}\kappa_{\Gamma_k} )^{-1}\right ]^\frac{1}{2}> M +40\sqrt{2}M\theta^\frac{1}{2}_\eta(n,p),\nonumber\\
    &\min_k \left [M^2 + \frac{\alpha_k}{24\tilde{s}_{k}(2\tilde{s}_{k}M^2\kappa^2_{\Gamma_k }+\kappa_{\Gamma_k})}\right ]^\frac{1}{2}>M + 40\sqrt{2}M\theta^\frac{1}{2}_\eta(n,p),\nonumber\\
    &\min_k A_k>40\sqrt{2}M\theta^\frac{1}{2}_\eta(n,p),\nonumber\\
    &\min_k 16\tilde{\sigma}_{k\bX}>40\sqrt{2}M\theta^\frac{1}{2}_\eta(n,p),\text{ and }\nonumber\\
    &\min_k 16\tilde{\sigma}_{k\bY}>40\sqrt{2}M\theta^\frac{1}{2}_\eta(n,p).
\end{align}
Let us consider the first inequality in \eqref{suff_ineqs}.
\begin{align}\label{ineq1}
    &\ \min_k \left [M^2 + (6\tilde{s}_{k}\kappa_{\Gamma_k} )^{-1}\right ]^\frac{1}{2}> M +40\sqrt{2}M\theta^\frac{1}{2}_\eta(n,p)\nonumber\\
    \Leftarrow &\ M^2 + \min_k  (6\tilde{s}_{k}\kappa_{\Gamma_k} )^{-1}> M^2 +80\sqrt{2}M^2\theta^\frac{1}{2}_\eta(n,p) + 3200M^2\theta_\eta(n,p)\nonumber\\
    \Leftarrow &\ \min_k  (6\tilde{s}_{k}\kappa_{\Gamma_k} )^{-1}> 80\sqrt{2}M^2\left (\theta^\frac{1}{2}_\eta(n,p) + 20\sqrt{2}\theta_\eta(n,p)\right )\nonumber\\
    \Leftarrow &\ \max_k  \tilde{s}_{k}\kappa_{\Gamma_k}<\frac{1}{ 480\sqrt{2}M^2\left (\theta^\frac{1}{2}_\eta(n,p) + 20\sqrt{2}\theta_\eta(n,p)\right )}.
\end{align}
Similarly, from the second inequality in \eqref{suff_ineqs} we obtain
\begin{align}\label{ineq2}
    &\ \min_k \left [M^2 + \frac{\alpha_k}{24\tilde{s}_{k}(2\tilde{s}_{k}M^2\kappa^2_{\Gamma_k }+\kappa_{\Gamma_k})}\right ]^\frac{1}{2}>M + 40\sqrt{2}M\theta^\frac{1}{2}_\eta(n,p)\nonumber\\
   \Leftarrow &\ M^2 +\min_k  \frac{\alpha_k}{24\tilde{s}_{k}(2\tilde{s}_{k}M^2\kappa^2_{\Gamma_k }+\kappa_{\Gamma_k})}>M^2 + 80\sqrt{2}M^2\theta^\frac{1}{2}_\eta(n,p)+3200M^2\theta_\eta(n,p)\nonumber\\
   \Leftarrow &\ \min_k  \frac{\alpha_k}{24\tilde{s}_{k}(2\tilde{s}_{k}M^2\kappa^2_{\Gamma_k }+\kappa_{\Gamma_k})}>80\sqrt{2}M^2\theta^\frac{1}{2}_\eta(n,p)+3200M^2\theta_\eta(n,p)\nonumber\\
   \Leftarrow &\ \max_k  \frac{24\tilde{s}_{k}(2\tilde{s}_{k}M^2\kappa^2_{\Gamma_k }+\kappa_{\Gamma_k})}{\alpha_k}<\frac{1}{80\sqrt{2}M^2\theta^\frac{1}{2}_\eta(n,p)+3200M^2\theta_\eta(n,p)}\nonumber\\
   \Leftarrow &\ \max_k  \frac{4\tilde{s}_{k}(2\tilde{s}_{k}M^2\kappa^2_{\Gamma_k }+\kappa_{\Gamma_k})}{\alpha_k}<\frac{1}{480\sqrt{2}M^2\left (\theta^\frac{1}{2}_\eta(n,p)+20\sqrt{2}\theta_\eta(n,p)\right )}.
\end{align}
Combining \eqref{ineq1}and \eqref{ineq2}, we have
\begin{align*}
    &\ \max\left \{\max_k \tilde{s}_k\kappa_{\Gamma_k},\ \max_k  \frac{4\tilde{s}_{k}\kappa_{\Gamma_k }(2M^2\tilde{s}_{k}\kappa_{\Gamma_k }+1)}{\alpha_k}\right \}<\frac{1}{480\sqrt{2}M^2\left (\theta^\frac{1}{2}_\eta(n,p)+20\sqrt{2}\theta_\eta(n,p)\right )}\nonumber\\
    \Leftarrow &\     \max\left \{\max_k \tilde{s}_k\kappa_{\Gamma_k},\ \frac{4\max_k  \tilde{s}_{k}\kappa_{\Gamma_k }(2M^2\max_k  \tilde{s}_{k}\kappa_{\Gamma_k }+1)}{\min_k\alpha_k}\right \}<\frac{1}{480\sqrt{2}M^2\left (\theta^\frac{1}{2}_\eta(n,p)+20\sqrt{2}\theta_\eta(n,p)\right )}\nonumber
\end{align*}
Since $4(2M^2\max_k\tilde{s}_{k}\kappa_{\Gamma_k }+1)>1$ and $\min_k\alpha_k<1$, 
\begin{align*}
    \max\left \{\max_k \tilde{s}_k\kappa_{\Gamma_k},\ \frac{4\max_k  \tilde{s}_{k}\kappa_{\Gamma_k }(2M^2\max_k  \tilde{s}_{k}\kappa_{\Gamma_k }+1)}{\min_k\alpha_k}\right \} = \frac{4\max_k  \tilde{s}_{k}\kappa_{\Gamma_k }(2M^2\max_k  \tilde{s}_{k}\kappa_{\Gamma_k }+1)}{\min_k\alpha_k}.
\end{align*}
Also, it follows from the definition of $\tilde{s}_{k}$ and $\kappa_{\Gamma_k }$ that
\begin{align*}
\max_k  \tilde{s}_{k}\kappa_{\Gamma_k } = \max_k  |\tilde{S}_{k}|\ \|\left (\Gamma_{k,\tilde{S}_{k}\tilde{S}_{k}}\right )^{-1}\|_{1,\infty}\geq \max_k \ \|\left (\Gamma_{k,\tilde{S}_{k}\tilde{S}_{k}}\right )^{-1} \|_1
\end{align*}
Now, \begin{align*}
    &\ \frac{4\max_k  \tilde{s}_{k}\kappa_{\Gamma_k }(2M^2\max_k  \tilde{s}_{k}\kappa_{\Gamma_k }+1)}{\min_k\alpha_k}<\frac{1}{480\sqrt{2}M^2\left (\theta^\frac{1}{2}_\eta(n,p)+20\sqrt{2}\theta_\eta(n,p)\right )}\nonumber\\
    \Leftrightarrow &\ \frac{\min_k\alpha_k}{4\max_k  \tilde{s}_{k}\kappa_{\Gamma_k }(2M^2\max_k  \tilde{s}_{k}\kappa_{\Gamma_k }+1)}>480\sqrt{2}M^2\left (\theta^\frac{1}{2}_\eta(n,p)+20\sqrt{2}\theta_\eta(n,p)\right ).
\end{align*}
Next, the third inequality in \eqref{suff_ineqs} is implied by the following condition:
\begin{align*}
     &\ \min_k A_k>40\sqrt{2}M\theta^\frac{1}{2}_\eta(n,p)\nonumber\\
     \Leftarrow &\ \min_k \frac{\alpha_k}{4-\alpha_k}>40\sqrt{2}\theta^\frac{1}{2}_\eta(n,p)\nonumber\\
      \Leftarrow &\ \min_k \frac{4}{4-\alpha_k}>1+40\sqrt{2}\theta^\frac{1}{2}_\eta(n,p)\nonumber\\
      \Leftarrow &\ \max_k \frac{4-\alpha_k}{4}<\frac{1}{1+40\sqrt{2}\theta^\frac{1}{2}_\eta(n,p)}\nonumber\\
      \Leftarrow &\ 1 - \min_k \frac{\alpha_k}{4}<\frac{1}{1+40\sqrt{2}\theta^\frac{1}{2}_\eta(n,p)}\nonumber\\
      \Leftarrow &\ \min_k \alpha_k>4(1 - \frac{1}{1+40\sqrt{2}\theta^\frac{1}{2}_\eta(n,p)})\nonumber\\
      \Leftarrow &\ \min_k \alpha_k>4\left (1 + \frac{1}{40\sqrt{2}\theta^\frac{1}{2}_\eta(n,p)}\right )^{-1}.
\end{align*}
The fourth inequality in \eqref{suff_ineqs} is implied by
\begin{align*}
    &\ \min_k 16\tilde{\sigma}_{k\bX}>40\sqrt{2}M\theta^\frac{1}{2}_\eta(n,p)\nonumber\\
    \Leftrightarrow &\ \min_k 16\max_j\ \frac{5}{2}\bS_{1k}(j,j)>40\sqrt{2}M\theta^\frac{1}{2}_\eta(n,p)\nonumber\\
    \Leftrightarrow &\ \min_k\max_j\ \bS_{1k}(j,j)>\sqrt{2}M\theta^\frac{1}{2}_\eta(n,p).
\end{align*}
Similarly, the fifth inequality in \eqref{suff_ineqs} holds if 
\begin{align*}
    \min_k\max_j\ \bS_{2k}(j,j)>\sqrt{2}M\theta^\frac{1}{2}_\eta(n,p).
\end{align*}
\begin{lemma}\label{finite_tune}
    If assumption A4 is satisfied, then $\max_k \lambda_k <\infty$.
\end{lemma}
\noindent \textbf{Proof of Lemma \ref{finite_tune}: }
Let us consider the first term in $\max_k\lambda_k$.
\begin{align*}
&\ \{128(\eta\ln{p}+\ln{4})\}^\frac{1}{2}\max_k \frac{2MG_{1k}}{A_k}\\
=&\ \{128(\eta\ln{p}+\ln{4})\}^\frac{1}{2}\max_k \left (\frac{8}{\alpha_k}-2\right )\left (\frac{\tilde{\sigma}_{k\bX}}{\sqrt{n_1}} + \frac{\tilde{\sigma}_{k\bY}}{\sqrt{n_2}}\right )\\
\leq &\ c_1\frac{(\eta\ln{p}+\ln{4})^\frac{1}{2}}{n^\frac{1}{2}}\left (\frac{4}{\min_k\alpha_k}-1\right )\max_k \left (\tilde{\sigma}_{k\bX} + \tilde{\sigma}_{k\bY}\right )\\
=&\ c_1\theta^\frac{1}{2}_\eta(n,p)\left (\frac{4}{\min_k\alpha_k}-1\right )\max_k \left (\tilde{\sigma}_{k\bX} + \tilde{\sigma}_{k\bY}\right )\\
\leq &\ c_2\max_k \left (\tilde{\sigma}_{k\bX} + \tilde{\sigma}_{k\bY}\right )\text{ [follows from A4]}\\
\leq &\ c_3\max_k \left (\ \max_j\ \bS_{1k}(j,j)+ \max_j\ \bS_{2k}(j,j)\right )\\
\leq &\ c_3\left (\max_k  \max_j\ \bS_{1k}(j,j)+\max_k  \max_j\ \bS_{2k}(j,j)\right )\leq 2c_3M.
\end{align*}
Now, consider the second term in $\max_k \lambda_k$.
\begin{align*}
    &\ \max_k \left (\{128(\eta\ln{p} + \ln{4})\}^\frac{1}{2}\widetilde{M}_k G_{2k} + MG_{1k} \widetilde{M}_k\right )\{128(\eta\ln{p} + \ln{4})\}^\frac{1}{2}\\
    =&\ \max_k\left (\{128(\eta\ln{p} + \ln{4})\}\widetilde{M}_k G_{2k} + MG_{1k} \widetilde{M}_k\{128(\eta\ln{p} + \ln{4})\}^\frac{1}{2}\right )\\
    \leq &\ c_1\frac{(\eta\ln{p} + \ln{4})}{n}\max_k\left (\max_j\ \bS_{1k}(j,j)\max_j\ \bS_{2k}(j,j)\right )\frac{\max_k \tilde{s}_{k} \kappa_{\Gamma_k}}{\min_k\alpha_k}(\max_k\tilde{s}_{k}\kappa_{\Gamma_k} +1)  \\
    &\ \ + c_2 \frac{(\eta\ln{p} + \ln{4})^\frac{1}{2}}{n^\frac{1}{2}}\max_k\left (\max_j\ \bS_{1k}(j,j)+\max_j\ \bS_{2k}(j,j)\right )\frac{\max_k\tilde{s}_{k} \kappa_{\Gamma_k}}{\min_k\alpha_k}(\max_k\tilde{s}_{k}\kappa_{\Gamma_k} +1)\\
   = &\ \left \{c_1M^2 \theta_{\eta}(n,p) + 2c_2M \theta^\frac{1}{2}_{\eta}(n,p)\right \}\frac{\max_k \tilde{s}_{k} \kappa_{\Gamma_k}}{\min_k\alpha_k}(\max_k\tilde{s}_{k}\kappa_{\Gamma_k} +1)<\infty \text{ [follows from A4]}.
\end{align*}\hfill\QEDB

\noindent \textbf{Proof of Theorem \ref{thm1}:} Recall the definition of $\bar{\sigma}_k$ given in \eqref{quantities}. If the assumptions A2-A4 are satisfied, then we have $$n>40\sqrt{2}\max_k\bar{\sigma}^{-2}_k(\eta\ln{p}+\ln{4}),\text{ where }\eta = \max\{\eta_1,\eta_2,\eta_3\}>2.$$ 
\begin{enumerate}
    \item Fix $k\in\{1,\ldots, T^\prime\}$ and observe that $n>40\sqrt{2}\bar{\sigma}^{-2}_k(\eta\ln{p}+\ln{4})$. The proof readily follows from Theorem 1 of \cite{yuan2017differential}.
    \item Due to Theorem 1 of \cite{yuan2017differential}, we have 
    \begin{align*}
    &\ \mathrm{P}\left [\|\hat{\widetilde{\bD}}_k - \widetilde{\bD}_k\|_\infty\leq M_{G_k}\theta_\eta^\frac{1}{2}(n,p)\right ]>1-\frac{2}{p^{\eta-2}},\\
    \text{i.e., }&\ \mathrm{P}\left [\|\hat{\widetilde{\bD}}_k - \widetilde{\bD}_k\|_\infty\leq \max_k M_{G_k}\theta_\eta^\frac{1}{2}(n,p)\right ]>1-\frac{2}{p^{\eta-2}},\\
     \text{i.e., }&\ \mathrm{P}\left [\|\hat{\widetilde{\bD}}_k - \widetilde{\bD}_k\|_\infty\geq \max_k M_{G_k}\theta_\eta^\frac{1}{2}(n,p)\right ]\leq \frac{2}{p^{\eta-2}}\ \text{ for all }k=1,\ldots, T^\prime.
    \end{align*}
    Now, $\max_k\|\hat{\widetilde{\bD}}_k - \widetilde{\bD}_k\|_\infty\geq \max_k M_{G_k}\theta_\eta^\frac{1}{2}(n,p)
        \Rightarrow \|\hat{\widetilde{\bD}}_k - \widetilde{\bD}_k\|_\infty\geq \max_k M_{G_k}\theta_\eta^\frac{1}{2}(n,p)\text{ for some }k$. Therefore,
    \begin{align*}
        &\ \mathrm{P}\left [\max_k\|\hat{\widetilde{\bD}}_k - \widetilde{\bD}_k\|_\infty\geq \max_k M_{G_k}\theta_\eta^\frac{1}{2}(n,p)\right ]\\
        \leq &\  \mathrm{P}\left [\|\hat{\widetilde{\bD}}_k - \widetilde{\bD}_k\|_\infty\geq \max_k M_{G_k}\theta_\eta^\frac{1}{2}(n,p)\text{ for some }k\right ]\\
        \leq &\  \sum\limits_{k=1}^{T^\prime}\mathrm{P}\left [\|\hat{\widetilde{\bD}}_k - \widetilde{\bD}_k\|_\infty\geq \max_k M_{G_k}\theta_\eta^\frac{1}{2}(n,p)\right ]\leq \frac{T}{p^{\eta-2}}.
    \end{align*}
    \item Again, due to Theorem 1 of \cite{yuan2017differential}, we have 
    \begin{align*}
    &\ \mathrm{P}\left [\|\hat{\widetilde{\bD}}_k - \widetilde{\bD}_k\|_\mathrm{F}\leq M_{G_k}\tilde{s}^\frac{1}{2}_k\theta_\eta^\frac{1}{2}(n,p)\right ]>1-\frac{2}{p^{\eta-2}},\ \text{ for all }k=1,\ldots, T^\prime.
    \end{align*}
    Following similar arguments detailed in $(b)$ , we have 
     \begin{align*}
        \mathrm{P}\left [\max_k\|\hat{\widetilde{\bD}}_k - \widetilde{\bD}_k\|_\infty\geq \max_k M_{G_k}\tilde{s}^\frac{1}{2}_k\theta_\eta^\frac{1}{2}(n,p)\right ]\leq \frac{T}{p^{\eta-2}}.
    \end{align*}
\end{enumerate} 
This completes the proof. \hfill\QEDB 
\begin{lemma}\label{sparsemean_lem1}
    Assume that $A_{\widetilde{\bS}k}\psi_k\epsilon_{\bD k}<1$. We have
    \begin{align*}
        \left \|\left (\widetilde{\bS}_{k, \Psi_k\Psi_k}\right ) ^{-1}-\left (\hat{\widetilde{\bS}}_{k, \Psi_k\Psi_k}\right )^{-1}\right \|_{1,\infty}\leq \frac{A^2_{\bS k}\psi_k\epsilon_{\bD k}}{1-A_{\bS k}\psi_k\epsilon_{\bD k}}\text { for all } k=1,\ldots, T^\prime.
    \end{align*}
\end{lemma}
\begin{proof}
    This lemma can be easily proved using the following observation:
\begin{align*}
    \left \|\left (\widetilde{\bS}_{k, \Psi_k\Psi_k}\right ) ^{-1}-\left (\hat{\widetilde{\bS}}_{k, \Psi_k\Psi_k}\right )^{-1}\right \|_{1,\infty}&\leq
    \left \|\left (\hat{\widetilde{\bS}}_{k, \Psi_k\Psi_k}\right ) ^{-1}\right \|_{1,\infty}
    \left \|\widetilde{\bS}_{k, \Psi_k\Psi_k}-\hat{\widetilde{\bS}}_{k, \Psi_k\Psi_k}\right \|_{1,\infty}
    \left \|\left (\widetilde{\bS}_{k, \Psi_k\Psi_k}\right ) ^{-1}\right \|_{1,\infty}\\
    &\leq \left (A_{\bS k}+ \left \|\left (\widetilde{\bS}_{k, \Psi_k\Psi_k}\right ) ^{-1}-\left (\hat{\widetilde{\bS}}_{k, \Psi_k\Psi_k}\right )^{-1}\right \|_{1,\infty}\right)A_{\bS k}\psi_k\epsilon_{\bD k}.
\end{align*}
\end{proof}
\begin{lemma}\label{sparsemean_lem2}
For $k=1,\ldots, T^\prime$, we have
    \begin{align*}
        |\hat{\gamma}_k-\gamma_k|_\infty\leq 8\epsilon_{\mu k} &\ + 2(\epsilon_{\bS k}+2M^*\epsilon_{\bD k}A_{2k}+ \tilde{s}_k\epsilon_{\bS k}\epsilon_{\bD k}A_{2k})|\widetilde{\bD}_k(\tilde{\mu}_{1k}-\tilde{\mu}_{2k})|_1\\
        &\ +2(2M + \epsilon_{\bS k})(A_{1k}+\tilde{s}_k\epsilon_{\bD k})\epsilon_{\mu k}.
    \end{align*}
\end{lemma}
\begin{proof}
    \begin{align*}
     |\hat{\gamma}_k-\gamma_k|_\infty\leq &\ 4|\hat{\Delta}_{\mu k}-\Delta_{\mu k}|_\infty + |(\Delta_{1k}-\Delta_{2k})\tilde{\bD}_k\Delta_{\mu k}|_\infty\\
     &\ + |(\hat{\widetilde{\bS}}_{1k}-\hat{\widetilde{\bS}}_{2k})(\hat{\tilde{\bD}}_k-\tilde{\bD}_k)\Delta_{\mu k}|_\infty + |(\hat{\widetilde{\bS}}_{1k}-\hat{\widetilde{\bS}}_{2k})\hat{\tilde{\bD}}_k(\hat{\Delta}_{\mu k}-\Delta_{\mu k})|_\infty
\end{align*}
Note that
\begin{align}\label{lem2_1}
    &|\hat{\Delta}_{\mu k}-\Delta_{\mu k}|_\infty\leq 2\epsilon_{\mu k}.
\end{align}
Next, consider the following identity:
\begin{gather}\label{iden1}
|\bM x|_1 = \sum_i\left |\sum_j M_{ij}x_j\right |\leq \sum_i\sum_j \left |M_{ij}x_j\right |\leq \sum_i\max_j|M_{ij}|\sum_j \left |x_j\right |
\end{gather}
where $\bM$ is a matrix and $x$ is a vector. 

Since
\begin{gather}\label{iden2}
    |\bM x|_\infty\leq \|\bM\|_\infty |x|_1,
\end{gather}
we have
\begin{align}\label{lem2_3}
|(\Delta_{1k}-\Delta_{2k})\widetilde{\bD}_k\Delta_{\mu k}|_\infty=\|\Delta_{1k}-\Delta_{2k}\|_\infty|\widetilde{\bD}_k\Delta_{\mu k}|_1\leq 2\epsilon_{\bS k}|\widetilde{\bD}_k(\tilde{\mu}_{1k}-\tilde{\mu}_{2k})|_1.
\end{align}
Also, consider another identity:
\begin{gather}\label{iden3}
     \|\bM_1\bM_2\|_\infty \leq \|\bM_1\|_{1,\infty}\|\bM_2\|_\infty
\end{gather}
It follows from \eqref{iden2} and \eqref{iden3} that
\begin{align}
    &\ |(\hat{\widetilde{\bS}}_{1k}-\hat{\widetilde{\bS}}_{2k})(\hat{\widetilde{\bD}}_k-\widetilde{\bD}_k)\Delta_{\mu k}|_\infty\nonumber \\
    \leq &\ \|(\widetilde{\bS}_{1k}-\widetilde{\bS}_{2k})(\hat{\widetilde{\bD}}_{k}-\widetilde{\bD}_{k})\|_\infty\ |\Delta_{\mu k}|_1+ \|(\Delta_{1k}-\Delta_{2k})(\hat{\widetilde{\bD}}_{k}-\widetilde{\bD}_{k})\|_\infty\ |\Delta_{\mu k}|_1\nonumber \\
    \leq &\ \|\widetilde{\bS}_{1k}-\widetilde{\bS}_{2k}\|_{1,\infty} \|\hat{\widetilde{\bD}}_{k}-\widetilde{\bD}_{k}\|_\infty |\Delta_{\mu k}|_1+ \|\Delta_{1k}-\Delta_{2k}\|_{1,\infty}\|\hat{\widetilde{\bD}}_{k}-\widetilde{\bD}_{k}\|_\infty |\Delta_{\mu k}|_1\nonumber \\
    &\ \leq 4M^*\epsilon_{\bD k}|\tilde{\mu}_{1k}-\tilde{\mu}_{2k}|_1 +2\tilde{s}_k\epsilon_{\bS k}\epsilon_{\bD k}|\tilde{\mu}_{1k}-\tilde{\mu}_{2k}|_1,\\ 
    &|(\hat{\widetilde{\bS}}_{1k}-\hat{\widetilde{\bS}}_{2k})\hat{\tilde{\bD}}_k(\hat{\Delta}_{\mu k}-\Delta_{\mu k})|_\infty\leq 2(2M+\epsilon_{\bS k})
\end{align}

\end{proof}
\begin{lemma}\label{ordered_dks_concent}
    If the assumptions A1-A4 are satisfied, then under the conditions and notations of Theorem \ref{thm1} we have
    \begin{enumerate}
        \item $\mathrm{P}\left [\max\limits_k|\hat{d}_k - d_k|\leq q_{\eta_4}(T, n,p)\right ]> 1-\frac{T}{p^{\eta_4-2}}$.
        \item $\mathrm{P}\left [\max\limits_k|\hat{d}_{(k)} - d_{(k)}|\leq q_{\eta_4}(T, n,p)\right ]> 1-\frac{T}{p^{\eta_4-2}}$.
    \end{enumerate}
\end{lemma}
\begin{proof}
    \begin{enumerate}
    \item 
Recall the definition of $d_k$ and observe that
\begin{align*}
     &\ |\hat{d}_k-d_k |=\left |\|\hat{\widetilde{\bD}}_k\|_\mathrm{F} - \|{\widetilde{\bD}}_k\|_\mathrm{F}\right |\leq \|\hat{\widetilde{\bD}}_k - {\widetilde{\bD}}_k\|_\mathrm{F},\\
     \text{i.e., }&\ \max\limits_k |\hat{d}_k-d_k |\leq \max\limits_k\|\hat{\widetilde{\bD}}_k - {\widetilde{\bD}}_k\|_\mathrm{F}.
\end{align*}
Therefore, $$ \mathrm{P}\left [\max\limits_k\|\hat{\widetilde{\bD}}_k - \widetilde{\bD}_k\|_\mathrm{F}\leq \max_k M_{G_k}\tilde{s}_k^\frac{1}{2}\theta_\eta^\frac{1}{2}(n,p) \right ]\leq \mathrm{P}\left [\max\limits_k |\hat{d}_k-d_k |\leq \max_k M_{G_k}\tilde{s}_k^\frac{1}{2}\theta_\eta^\frac{1}{2}(n,p) \right ].$$
The proof follows from Theorem \ref{thm1}(b).
\item Let $(u_1,\ldots, u_{T^\prime})^\top$ and $(v_1,\ldots, v_{T^\prime})^\top$ denote two vectors in $\mathbb{R}^{T^\prime}$. Then, for any $1\leq k\leq T^\prime,$ we have $|u_{(k)}-v_{(k)}|\leq |u_i - v_j|$ for $1\leq i,j\leq T^\prime,$ where $i$ and $j$ are such that $u_i\geq u_{(k)}$ and $v_j\leq v_{(k)}.$ There are $(T^\prime-k+1)$ and $k$  such choices for $i$ and $j$, respectively. It follows from the {\it pigeon-hole principle} that for each $1\leq k \leq T^\prime$ there exists at least one $l$ satisfying $1\leq l \leq T^\prime$ such that $|u_{(k)}-v_{(k)}|\leq |u_l - v_l|$ (see \cite{wainwright_2019}). Therefore,
\begin{align}\label{php}
&\ |u_{(k)}-v_{(k)}|\leq \max\limits_{1\leq l\leq T^\prime}|u_l - v_l| \text{ for all }1\le k\le T^\prime\nonumber \\
\Rightarrow &\ \max\limits_{1\leq k\leq T^\prime}|u_{(k)}-v_{(k)}|\leq \max\limits_{1\leq l\leq T^\prime}|u_l - v_l|.
\end{align}
Using this result for the vectors $(\hat{d}_{1},\ldots, \hat{d}_{T^\prime})^\top$ and $(d_{1},\ldots, d_{T^\prime})^\top ,$ we obtain the following:
\begin{align*}
&\ \max\limits_{k}|\hat{d}_{(k)} - {d}_{(k)}|> \max_k M_{G_k}\tilde{s}_k^\frac{1}{2}\theta^\frac{1}{2}_\eta(n,p)\Rightarrow \max\limits_{k}|\hat{d}_{k} - {d}_{k}|> \max_k M_{G_k}\tilde{s}_k^\frac{1}{2}\theta^\frac{1}{2}_\eta(n,p)\\
\text{i.e., } &\ P\left [\max\limits_{k}|\hat{d}_{(k)} - {d}_{(k)}|> \max_k M_{G_k}\tilde{s}_k^\frac{1}{2}\theta^\frac{1}{2}_\eta(n,p)\right ]
\leq P\left [ \max\limits_{k}|\hat{d}_{k} - {d}_{k}|> \max_k M_{G_k}\tilde{s}_k^\frac{1}{2}\theta^\frac{1}{2}_\eta(n,p)\right ].
\end{align*}
The proof follows from part (a).
\end{enumerate}

\end{proof}
\begin{lemma}\label{surescreen_d}
    If assumptions A1-A4 are satisfied, then 
    \begin{align*}
    \mathrm{P}\left [\hat{r}^d_{T_{-\bD}}<\max\limits_{1\leq k\leq T_{\bD}-1}\hat{r}^d_{T_{-\bD}+k}\right ]<\frac{T}{p^{\eta_4-2}}.
    \end{align*}
\end{lemma}
\begin{proof}
    \begin{align}\label{maxrat}
    &\ \max\limits_k|\hat{d}_{(k)} - d_{(k)}|\leq q_{\eta_4}(T, n,p)\nonumber \\
    \Rightarrow &\ \hat{d}_{(T_{0})}<q_{\eta_4}(T, n,p),\text{ and}\nonumber\\
    &\hspace{1cm}{d}_{(T_{-\bD}+k)}-q_{\eta_4}(T, n,p) <\hat{d}_{(T_{-\bD}+k)}<{d}_{(T_{-\bD}+k)}+q_{\eta_4}(T, n,p)\ \forall k = 1,\ldots,  T_{\bD}-1,\\
    \Rightarrow &\ \frac{\hat{d}_{(T_{-\bD}+1)}}{\hat{d}_{(T_{0})}}>\frac{{d}_{(T_{-\bD}+1)}-q_{\eta_4}(T, n,p)}{q_{\eta_4}(T, n,p)},\text{ and }\nonumber \\
    &\hspace{1cm}\frac{\hat{d}_{(T_{-\bD}+k+1)}}{\hat{d}_{(T_{-\bD}+k)}}<\frac{{d}_{(T_{-\bD}+k+1)}+q_{\eta_4}(T, n,p)}{{d}_{(T_{-\bD}+k)}-q_{\eta_4}(T, n,p)}\forall k = 1,\ldots,  T_{\bD}-1,\nonumber\\
    \Rightarrow &\ \hat{r}^d_{(T_{-\bD})}>\frac{{d}_{(T_{-\bD}+1)}-q_{\eta_4}(T, n,p)}{q_{\eta_4}(T, n,p)},\text{ and }\nonumber \\
    &\hspace{1cm}\max \limits_{1\leq k\leq T_{\bD}-1}\frac{\hat{d}_{(T_{-\bD}+k+1)}}{\hat{d}_{(T_{-\bD}+k)}}<\max \limits_{1\leq k\leq T_{\bD}-1}\frac{{d}_{(T_{-\bD}+k+1)}+q_{\eta_4}(T, n,p)}{{d}_{(T_{-\bD}+k)}-q_{\eta_4}(T, n,p)}
\end{align}

Fix a $k\in\{1,\ldots, T^\prime\}$. It follows from A7 that 
\begin{align*}
    &\ \frac{d_{(T_{-\bD}+k+1)}}{d_{(T_{-\bD}+k)}}<\frac{d_{(T_{-\bD}+1)}}{3q_{\eta_5}(T, n,p)}\nonumber\\
    \Rightarrow &\ 3q_{\eta_5}(T,n,p)d_{(T_{-\bD}+k+1)}<d_{(T_{-\bD}+1)}d_{(T_{-\bD}+k)}\nonumber\\
    \Rightarrow &\ q_{\eta_5}(T,n,p)(d_{(T_{-\bD}+k+1)}+d_{(T_{-\bD}+k)} + d_{(T_{-\bD}+1)})<d_{(T_{-\bD}+1)}d_{(T_{-\bD}+k)}\nonumber\\
     \Rightarrow &\ q_{\eta_4}(T, n,p)\left ({d}_{(T_{-\bD}+1)}+{d}_{(T_{-\bD}+k)}+{d}_{(T_{-\bD}+k+1)}\right ) < {d}_{(T_{-\bD}+1)}{d}_{(T_{-\bD}+k)}\nonumber\\
     \Rightarrow &\ {d}_{(T_{-\bD}+k+1)}q_{\eta_4}(T, n,p) < {d}_{(T_{-\bD}+1)}{d}_{(T_{-\bD}+k)} - q_{\eta_4}(T, n,p)\left ({d}_{(T_{-\bD}+1)}+{d}_{(T_{-\bD}+k)}\right )\nonumber\\
     \Rightarrow &\ {d}_{(T_{-\bD}+k+1)}q_{\eta_4}(T, n,p)+ q^2_{\eta_4}(T, n,p)\nonumber\\
     &\hspace{1cm}< {d}_{(T_{-\bD}+1)}{d}_{(T_{-\bD}+k)} - q_{\eta_4}(T, n,p)\left ({d}_{(T_{-\bD}+1)}+{d}_{(T_{-\bD}+k)}\right )+q^2_{\eta_4}(T, n,p)\nonumber\\
     \Rightarrow &\ q_{\eta_4}(T, n,p)\left ({d}_{(T_{-\bD}+k+1)}+ q_{\eta_4}(T, n,p)\right )\nonumber\\
     &\hspace{1cm}< \left ({d}_{(T_{-\bD}+k)}- q_{\eta_4}(T, n,p)\right )\left ({d}_{(T_{-\bD}+1)}- q_{\eta_4}(T, n,p)\right )\nonumber\\
     \Rightarrow &\ \frac{{d}_{(T_{-\bD}+k+1)}+q_{\eta_4}(T, n,p)}{{d}_{(T_{-\bD}+k)}-q_{\eta_4}(T, n,p)}<\frac{{d}_{(T_{-\bD}+1)}-q_{\eta_4}(T, n,p)}{q_{\eta_4}(T, n,p)}\nonumber
\end{align*}
Therefore,
\begin{align}\label{maxrat1}
    \max \limits_{1\leq k\leq T_{\bD}-1}\frac{{d}_{(T_{-\bD}+k+1)}+q_{\eta_4}(T, n,p)}{{d}_{(T_{-\bD}+k)}-q_{\eta_4}(T, n,p)}<\frac{{d}_{(T_{-\bD}+1)}-q_{\eta_4}(T, n,p)}{q_{\eta_4}(T, n,p)}
\end{align}
Therefore, if A7 is satisfied, then combining \eqref{maxrat} and \eqref{maxrat1}, we obtain
\begin{align*}
    &\ \max\limits_k|\hat{d}_{(k)} - d_{(k)}|\leq q_{\eta_4}(T, n,p)\Rightarrow\max \limits_{1\leq k\leq T_{\bD}-1}\frac{\hat{d}_{(T_{-\bD}+k+1)}}{\hat{d}_{(T_{-\bD}+k)}}<\hat{r}^d_{(T_{-\bD})}\nonumber\\
    \text{i.e., }&\ \mathrm{P}\left [\max\limits_k|\hat{d}_{(k)} - d_{(k)}|\leq q_{\eta_4}(T, n,p)\right ]\leq \mathrm{P}\left [\max \limits_{1\leq k\leq T_{\bD}-1}\frac{\hat{d}_{(T_{-\bD}+k+1)}}{\hat{d}_{(T_{-\bD}+k)}}<\hat{r}^d_{(T_{-\bD})}\right ].
\end{align*}
The proof follows from Lemma \ref{ordered_dks_concent}(b).
\end{proof}

\noindent\textbf{Proof of Theorem \ref{thm4} :} Let us assume that $\Omega^\bD_T$ is not a subset of $\hat{\Omega}^\bD_T$. Now, $\Omega^\bD_T\not\subseteq \hat{\Omega}^\bD_T$ means that the set $\{\hat{d}_k\leq \hat{d}_{(T_{0})}\text{ for some }k\text{ with }\omega_k\in \Omega^\bD_T\}$ is non-empty. Thus,
\begin{align*}
&\ P\left [\Omega^\bD_T\not\subseteq \hat{\Omega}^\bD_T\right ] \\
=&\ P\left[\hat{d}_k\leq \hat{d}_{(\hat{T}_{0})}\text{ for some }k\text{ with }\omega_k\in \Omega^\bD_T\right ] \\
\leq &\ P\left[\hat{d}_k\leq \hat{d}_{(\hat{T}_{0})}\text{ for some }k\text{ with }\omega_k\in \Omega^\bD_T,\max\limits_{1\leq l\leq T^{\prime}}|\hat{d}_{(l)} - d_{(l)}|\leq q_{\eta_4}(T, n,p) \right]\\
&\ +P\left [\max\limits_{1\leq l\leq T^{\prime}}|\hat{d}_{(l)} - d_{(l)}|> q_{\eta_4}(T, n,p) \right ]\\
\leq &\ \sum\limits_{k:\omega_k\in\Omega^\bD_T}P\left [\hat{d}_k\leq \hat{d}_{(\hat{T}_{0})},\max\limits_{1\leq l\leq T^{\prime}}|\hat{d}_{(l)} - d_{(l)}|\leq q_{\eta_4}(T, n,p) \right ]+P\left [\max\limits_{1\leq l\leq T^{\prime}}|\hat{d}_{(l)} - d_{(l)}|> q_{\eta_4}(T, n,p) \right ].
\end{align*}

Using Lemma \ref{ordered_dks_concent}(b), we have $P\left [\max\limits_{1\leq l\leq T^{\prime}}|\hat{d}_{(l)} - d_{(l)}|> q_{\eta_4}(T, n,p) \right ]<\frac{T}{p^{\eta_4-2}}$. Consequently,
\begin{align}\label{e2}
P\left [\Omega^\bD_T\not\subseteq \hat{\Omega}^\bD_T\right ]\le \sum\limits_{k:\omega_k\in\Omega^\bD_T}P\left [\hat{d}_k\leq \hat{d}_{(\hat{T}_{0})},\max\limits_{1\leq l\leq T^{\prime}}|\hat{d}_{(l)} - d_{(l)}|\leq q_{\eta_4}(T, n,p) \right ]+\frac{T}{p^{\eta_4-2}}.
\end{align}
Under the conditions in Lemma \ref{surescreen_d}, we have \begin{align}\label{nref17}
&\max_{1\leq l\leq T^{\prime}}|\hat{d}_{(l)} - d_{(l)}|\leq q_{\eta_4}(T, n,p)\Rightarrow \hat{r}^d_{T_{-\bD}} > \max\limits_{1\leq l\leq (T_{\bD} -1)}\hat{r}^d_{T_{-\bD}+l}\nonumber\\
\Rightarrow &\  \argmax\limits_{1\leq l \leq (T^\prime -1)}\hat{r}^d_{l}\leq {T_{-\bD}}\Rightarrow \hat{T}_{-\bD}\leq {T_{-\bD}}\Rightarrow  \hat{d}_{\hat{T}_{-\bD}}\leq \hat{d}_{{T_{-\bD}}}.
\end{align}

\noindent Therefore, it follows from \eqref{e2} that
\begin{align}\label{sspproof}
&\ P[\Omega^\bD_T\not\subseteq \hat{\Omega}^\bD_T]\nonumber \\
\leq &\sum\limits_{k:\omega_k\in\Omega^\bD_T}P\left [\hat{d}_k\leq \hat{d}_{(\hat{T}_{0})},\max_{1\leq l\leq T^{\prime}}|\hat{d}_{(l)} - d_{(l)}|\leq q_{\eta_4}(T, n,p) \right ]+\frac{T}{p^{\eta_4-2}}\nonumber \\
\leq &\sum\limits_{k:\omega_k\in\Omega^\bD_T}P\left [\hat{d}_k\leq \hat{d}_{(T_{0})},\max_{1\leq l\leq T^{\prime}}|\hat{d}_{(l)} - d_{(l)}|\leq q_{\eta_4}(T, n,p) \right ]+\frac{T}{p^{\eta_4-2}}.
\end{align}
Fix a $k$ with $\omega_k\in\Omega^\bD_T$, and observe that
\begin{align}\label{sscreen1}
&\ P\left [\hat{d}_k\leq \hat{d}_{(T_{0})},\max_{1\leq l\leq T^{\prime}}|\hat{d}_{(l)} - d_{(l)}|\leq q_{\eta_4}(T, n,p) \right ]\nonumber \\
\leq &\ P\left [\hat{d}_k\leq \hat{d}_{(T_{0})},\max_{1\leq l\leq T^{\prime}}|\hat{d}_{l} - d_{l}|\leq q_{\eta_4}(T, n,p),\max_{1\leq l\leq T^{\prime}}|\hat{d}_{(l)} - d_{(l)}|\leq q_{\eta_4}(T, n,p) \right ]\nonumber \\
\leq&\ P\left [{d}_k - q_{\eta_4}(T, n,p)\leq \hat{d}_k\leq \hat{d}_{(T_{0})}\leq q_{\eta_4}(T, n,p)\right ]\nonumber \\
\leq &\ \mathbb{I}[{d}_k - q_{\eta_4}(T, n,p)\leq q_{\eta_4}(T, n,p)]\nonumber\\
=&\ \mathbb{I}[{d}_k \leq 2 q_{\eta_4}(T, n,p)]\leq \mathbb{I}[{d}_k \leq 2 q_{\eta_5}(T, n,p)]=0\text{ [follows from A5]}.
\end{align}
Therefore, combining \eqref{sspproof} and \eqref{sscreen1}, we obtain
$$P[\Omega^\bD_T\not\subseteq \hat{\Omega}^\bD_T]\leq \frac{T}{p^{\eta_4-2}}.$$
This completes the proof.\hfill \QEDB

\end{document}